\documentclass[lettersize,journal]{IEEEtran}
\usepackage{amsmath,amsfonts}
\usepackage{algorithmic}
\usepackage{algorithm}
\usepackage{array}
\usepackage{textcomp}
\usepackage{stfloats}
\usepackage{url}
\usepackage{verbatim}
\usepackage{graphicx}
\usepackage{cite}
\usepackage{amssymb,amsthm}
\usepackage{epsfig,subfigure}
\usepackage{epsfig} % for postscript graphics files
\usepackage{amsmath} % assumes amsmath package installed
\usepackage{bm} 
\usepackage{amsthm} % Required for theorem environments
\usepackage{siunitx}  % 数字对齐
\usepackage{booktabs}
\usepackage{amscd} 
\usepackage{mathtools} % <-- Add this line
\usepackage{mathrsfs}
\usepackage{psfrag}
\usepackage{makecell}
\usepackage{soul}
\usepackage{color}
\usepackage{etoolbox} 
\usepackage{xcolor}

\def\tsc#1{\csdef{#1}{\textsc{\lowercase{#1}}\xspace}}
\tsc{WGM}
\tsc{QE}
\newtheorem{definition}{Definition}
\newtheorem{assumption}{Assumption}
\newcommand{\rr}{\mathbb{R}}
\newtheorem{theorem}{Theorem} 
\newtheorem{lemma}{Lemma}
\newtheorem{remark}{Remark}
\newtheorem{corollary}{Corollary}  % 定义推论环境
\makeatletter
\pretocmd\@bibitem{\color{black}\csname keycolor#1\endcsname}{}{\fail}
\newcommand\citecolor[1]{\@namedef{keycolor#1}{\color{blue}}}
\def\0{{\bf 0}}

\def\E{\mathbb{R}}
\begin{document}
\bstctlcite{IEEEexample:BSTcontrol}

\title{Safe and Nonconservative Contingency Planning for Autonomous Vehicles via Online Learning-Based Reachable Set Barriers}
% Recursively Feasible Contingency Planning for Safe and Non-Conservative Autonomous Driving via Online Learning Reachable Set Barriers
\author{Rui Yang, Lei Zheng, Shuzhi Sam Ge, \textit{Fellow, IEEE,} and Jun Ma, \textit{Senior Member, IEEE} 
  \thanks{
    Rui Yang and Jun Ma are with The Hong Kong University of Science and Technology, China (e-mail: ryang253@connect.hkust-gz.edu.cn; jun.ma@ust.hk).}
% \thanks{Jun Ma is with the Robotics and Autonomous Systems Thrust, The Hong Kong University of Science and Technology (Guangzhou), Guangzhou 511453, China, and also with the Division of Emerging Interdisciplinary Areas, The Hong Kong University of Science and Technology, Hong Kong SAR, China (e-mail: jun.ma@ust.hk).}  
\thanks{Lei Zheng and Shuzhi Sam Ge are with the Department of Electrical and Computer Engineering, National University of Singapore, Singapore 117576 (e-mail: zack.zheng@nus.edu.sg; samge@nus.edu.sg).}
 % \thanks{This work has been submitted to the IEEE for possible publication. Copyright may be transferred without notice, after which this version may no longer be accessible.} 
} 


% % column for its text to clear the IEEEpubid mark.

\maketitle

\begin{abstract}
% Autonomous vehicles must navigate dynamically uncertain environments while balancing the safety and driving efficiency. This challenge is exacerbated by the unpredictable nature of surrounding human-driven vehicles (HVs) and perception inaccuracies, which require planners to adapt to evolving uncertainties while maintaining safe trajectories. 
% Overly conservative planners degrade driving efficiency, while deterministic approaches may encounter serious issues and risks of failure when faced with sudden and unexpected maneuvers. 
% To address these issues, we propose a real-time contingency trajectory optimization framework in this paper. By employing event-triggered online learning of HV control-intent sets, our method dynamically quantifies \textcolor{blue}{multimodal} HV uncertainties and refines the forward reachable set (FRS) incrementally. Crucially, we enforce invariant safety through FRS-based barrier constraints that ensure safety without reliance on accurate trajectory prediction of HVs.
% These constraints are embedded in contingency trajectory optimization and solved efficiently through consensus alternative direction method of multipliers (ADMM). The system continuously adapts to the uncertainties in HV behaviors, preserving feasibility and safety without resorting to excessive conservatism. 
% High-fidelity simulations on highway and urban scenarios, as well as a series of real-world experiments demonstrate significant improvements in driving efficiency and passenger comfort while maintaining safety under uncertainty.
% The project page is available at \url{https://pathetiue.github.io/frscp.github.io/}.

Autonomous vehicles must navigate dynamically uncertain environments while balancing safety and efficiency.
This challenge is exacerbated by unpredictable human-driven vehicle (HV) behaviors and perception inaccuracies, necessitating planners that adapt to evolving uncertainties while maintaining safe trajectories. 
Overly conservative planning degrades driving efficiency, while deterministic methods risk failure in unexpected scenarios. 
To address these issues, we propose a real-time contingency trajectory optimization framework. 
Our method employs event-triggered online learning of HV control-intent sets to dynamically quantify multimodal HV uncertainties and incrementally refine their forward reachable sets (FRSs). 
Crucially, we enforce invariant safety through FRS-based barrier constraints that ensure safety without reliance on accurate trajectory prediction.
These constraints are seamlessly embedded in contingency trajectory optimization and solved efficiently through consensus alternating direction method of multipliers (ADMM). 
The system continuously adapts to HV behavioral uncertainties, preserving feasibility and safety without excessive conservatism. 
High-fidelity simulations on highway and urban scenarios, along with a series of real-world experiments, demonstrate significant improvements in driving efficiency and passenger comfort while maintaining safety under uncertainty.
The project page is available at \url{https://pathetiue.github.io/frscp.github.io/}.

\end{abstract}

\begin{IEEEkeywords}
Autonomous vehicles, trajectory optimization, reachability analysis, contingency planning.
\end{IEEEkeywords}

\vspace{-2mm}
\section{Introduction}
\IEEEPARstart{E}{nsuring} the safety of autonomous vehicles in dynamic, uncertain environments is paramount. This requires strict safety guarantees without unnecessarily compromising driving efficiency~\cite{chen2022milestones, zheng2024barrier}. A core challenge lies in accurately modeling and adapting to the uncertain, time-varying behaviors of surrounding human-driven vehicles (HVs), compounded by perception inaccuracies~\cite{benciolini2024active, mustafa2024racp}. Unanticipated contingency events, such as abrupt lane changes or sudden acceleration/deceleration, can force disruptive maneuvers and undermine planning feasibility, posing significant risks~\cite{zhang2025automated, zhang2024interaction, zheng2024}. Like expert drivers who adapt their caution based on perceived intent (e.g., giving more space to aggressive vehicles), autonomous vehicles must interpret and respond to evolving uncertainties online to adaptively maintain safety and efficiency. Yet achieving such a flexible balance remains elusive in practice: planners that ignore uncertainty can underestimate risks, whereas worst-case formulations become so conservative that efficiency suffers \cite{althoff2016set, gharavi2024proactive, ge2005queues}. Consequently, ensuring the persistent existence of safe trajectories while maintaining high driving efficiency under uncertainty remains an unresolved problem, particularly when adaptive responsiveness to uncertainties is essential.

Stochastic optimization methods, such as chance constrained approaches, address this problem by permitting small, predefined collision probabilities~\cite{zhu2019chance, nair2024predictive}. However, exact chance constraint evaluation is computationally intractable~\cite{de2023scenario}, leading to common approximations with Gaussian uncertainty models~\cite{zhang2024interaction, brudigam2021stochastic}. Recent advances, such as chance-constrained nonlinear model predictive control (NMPC) for multimodal obstacle behaviors~\cite{ren2024recursively}, ensure recursive feasibility and safety across the entire planning
horizon. Nevertheless, these methods still face fundamental limitations: Gaussian approximations struggle to capture the inherently non-Gaussian nature of real-world driving behaviors~\cite{wang2020non}, and even minor constraint violations can lead to hazardous outcomes in unexpected situations. Scenario optimization offers a sampling-based alternative~\cite{de2021scenario}, but demands large sample sets to address low-probability, high-risk events, which hinder its practical application in the real world.

To overcome these limitations, robust optimization enforces strict safety by requiring constraint satisfaction under all possible uncertainty realizations within bounded sets \cite{ben1998robust, carrizosa2024safe, rahimi2022robust}. This is typically achieved using forward reachability analysis in motion planning~\cite{seo2022real, manzinger2020using}. For example, prior works \cite{althoff2016set, koschi2020set} compute the forward reachable sets (FRSs) for traffic participants while accounting for road network constraints and worst-case control capabilities. Similarly, a robust NMPC framework is designed using HV FRS to guarantee safety under motion uncertainties of the HV~\cite{batkovic2023experimental}. To ensure recursive feasibility, methods such as~\cite{skibik2023mpc, pek2018computationally} maintain fail-safe trajectories over infinite horizons, allowing the ego vehicle (EV) to fall back to emergency maneuvers when needed\cite{pek2020fail}. In \cite{benciolini2024safe}, a safe backup trajectory is computed using FRS and enforced if the primary stochastic NMPC optimization becomes infeasible. However, sudden transitions to backup trajectories may degrade passenger comfort.

While reachability-based methods offer robust safety guarantees, control barrier function (CBF) provides an alternative approach through forward invariance~\cite{ejaz2024trust}. The forward invariance of safety sets is closely related to the recursive feasibility for safe trajectories~\cite{fang2022recursive, kerrigan2000invariant}, which can be naturally encoded via barrier function constraints \cite{zeng2021safety}. Recent advances integrate reachability analysis with CBFs to certify forward-invariant safety despite uncertainty \cite{choi2021robust, kumar2023fast}. For instance, FRS predictions are employed to construct CBFs to handle uncertainties in HV behaviors~\cite{kim2024safety}. While these CBF-based methods provide strong safety guarantees, they often remain overly conservative, especially in dense traffic scenarios, where worst-case assumptions lead to infeasible or inefficient motion plans.

To reduce conservatism without sacrificing safety, recent research explores learning-based approaches to refine uncertainty estimates. Uncertainty sets are constructed as subsets of worst-case forward reachable space to enable safer yet less conservative planning~\cite{gao2021risk, wang2024reachability}. When sufficient prior data is available, FRS predictions can be refined through online updates~\cite{bansal2020hamilton, driggs2018robust}. Learning techniques such as self-supervised neural networks\cite{xiang2024convex}, Gaussian processes~\cite{cao2021estimating, devonport2020data}, and active learning~\cite{chakrabarty2020active} show promise, yet often rely on extensive offline training or face challenges in real-time adaptation for multiple HVs. For online learning efficiency, polytopic control-intent sets are learned to predict FRSs of dynamic obstacles~\cite{zhou2025robusta}, while ellipsoidal representations are employed to accelerate computation~\cite{zhou2025robustb}. However, most of these works do not directly address the recursive feasibility problem under uncertainty. Furthermore, planning solely based on the FRSs of HVs can still be conservative, as it requires a single trajectory to perform well under all possible behaviors, which inherently couples safety with efficiency and leads to excessive conservatism.

Complementing these approaches, contingency planning offers a balanced approach by optimizing nominal behavior while maintaining flexibility through backup plans~\cite{hardy2013contingency, chen2022interactive}. These methods jointly optimize over a primary horizon for performance and a parallel horizon to handle contingencies~\cite{zhan2016non, mustafa2024racp}. For example, contingency planning is applied to account for model uncertainties and reduce the conservatism of worst-case FRS prediction in~\cite{chen2023invariant}. Nevertheless, the non-convex nature of safety requirements in multiple horizons and inherent nonholonomic kinematic constraints pose significant computational challenges for real-time implementation~\cite{ma2022alternating}

Recent advances in trajectory optimization for autonomous driving have employed distributed optimization techniques to address the computational bottlenecks. Specifically, the alternating direction method of multipliers (ADMM) is employed to maintain feasibility under constraints while achieving real-time performance\cite{adajania2022multi, ma2022alternating}. Particularly effective for separable biconvex optimization problems, ADMM-based approaches leverage parallel computing capabilities to facilitate high computational efficiency, enabling rapid response to changing environments~\cite{zheng2024barrier, zheng2025occlusion}.

\begin{figure}[tb]
\begin{center}
\includegraphics[width=0.85\columnwidth]{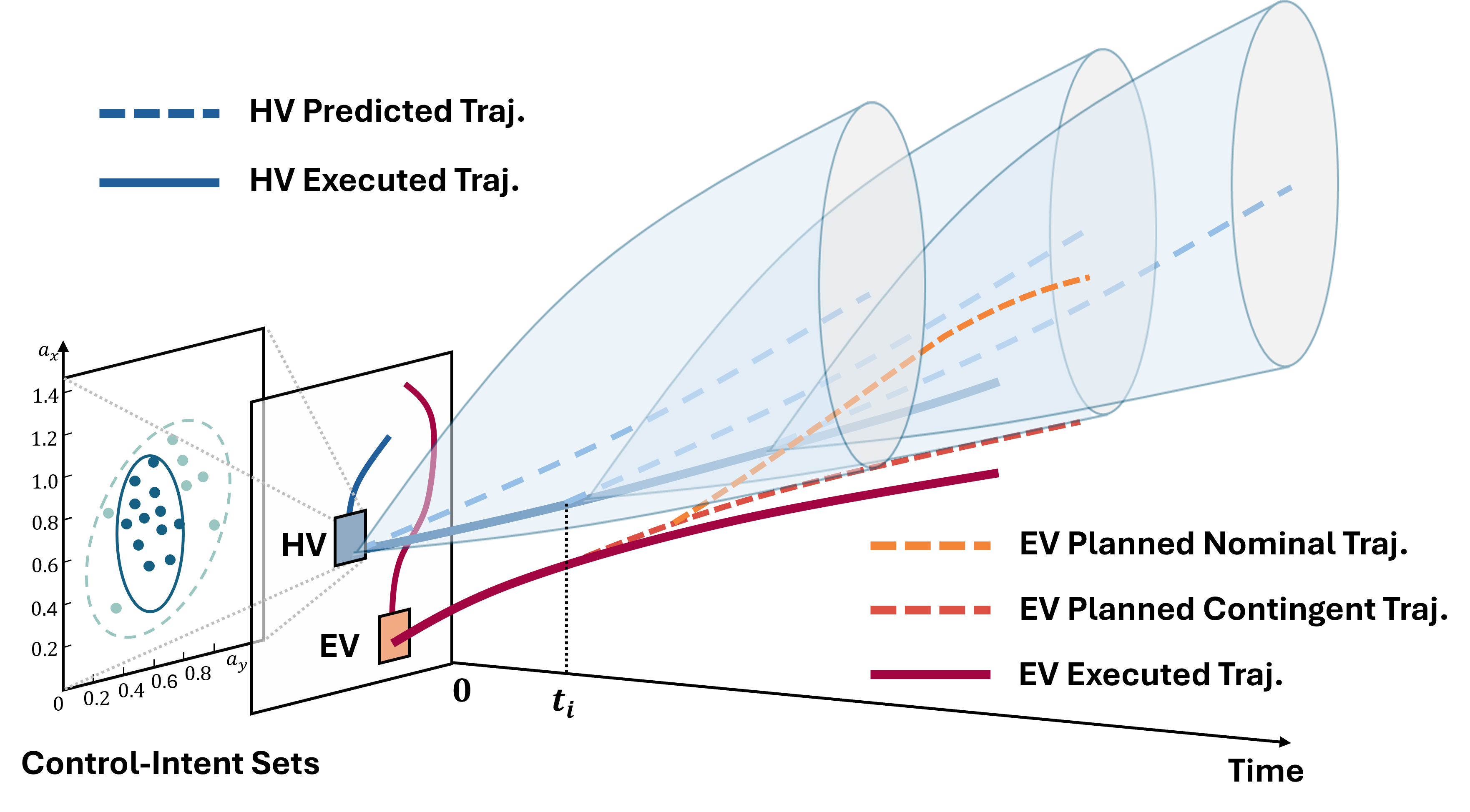}    \vspace{-3mm}
\caption{  
Overview of the proposed safe contingency trajectory planning framework. The method incrementally refines the control-intent sets of HVs through event-triggered online learning, enabling more accurate FRS predictions (light blue tubes) that adapt to evolving uncertainties in HV behaviors. At each time step $t_i$, the planner jointly optimizes a nominal performance-driven trajectory and a contingency trajectory enforced by FRS-based barrier constraints, ensuring safety while maintaining comfort and balanced efficiency.
} \vspace{-7mm}
\label{fig:architecture}
\end{center}
\end{figure} 

In this study, we propose a real-time contingency trajectory optimization framework that features an algorithmic co-design of constraint formulation, planning architecture, and numerical solving.
% that enforces safety through FRS-based barrier functions, with dynamic refinement of HV control-intent sets using event-triggered online learning.
The overview of the proposed framework is shown in Fig.~\ref{fig:architecture}. 
This approach ensures feasible, safe trajectories while adaptively managing uncertainty to avoid excessive conservatism. Specifically, the main contributions of this work are summarized as follows:
\begin{itemize}
        
        \item Propose a novel safety-preserving framework that integrates FRS with discrete barrier functions to ensure robust forward invariance of the safety set. This framework dynamically refines safety constraints via event-triggered online learning of HV control-intent sets, thereby addressing behavioral uncertainty and ensuring collision-free maneuvers without reliance on accurate trajectory predictions.
        
         \item Develop a contingency trajectory optimization scheme that ensures safe trajectories with recursive feasibility. This framework co-optimizes a nominal performance-driven trajectory and a contingency trajectory with FRS-based barrier constraints, which ensures planning outcomes toward regions where safe maneuvers exist, enabling defensive driving while maintaining efficiency.
    
        \item Exploit the specific bi-convex structure of the barrier constraints to decompose the non-convex contingency planning problem into tractable QP subproblems via consensus ADMM. This algorithmic co-design enables efficient computation of feasible trajectories, facilitating rapid response to uncertain driving conditions.
        
        \item Validate the proposed approach through high-fidelity simulations on synthetic and real-world traffic datasets and hardware experiments. The results demonstrate its ability to balance safety and efficiency while improving passenger comfort, with hardware tests further verifying its real-time deployability in dynamic traffic environments.
\end{itemize}

The remainder of this paper is structured as follows: Section~\ref{sec:problem} introduces the problem statement. Section~\ref{sec:Learning_Reachable_Set} discusses event-triggered online learning of HV control-intent sets and FRS prediction, followed by FRS-based safety barrier design. Section~\ref{sec:optimization_problem} describes the contingency planning framework and optimization scheme with consensus ADMM, with supporting theoretical analysis. The proposed trajectory planning approach is evaluated in Section~\ref{sec:Results}, and a conclusion is drawn in Section~\ref{sec:Conclusions}.

%=============================
\vspace{-5mm}
\section{Problem Statement}
\label{sec:problem} 
%=============================
\subsection{System Modeling}
\label{subsec:modeling} 
We model the motion of the $h$-th HV's  geometric center with a discrete-time integrator for  $h\in\mathcal{I}_0^{M-1}$, where $M$ is the total number of HVs in the considered scenario. This formulation facilitates computational tractability while maintaining sufficient accuracy for reachability analysis. With sampling time $\delta T$, the dynamics model of the $h$-th HV at time instant $t_i$ is given by:
\begin{equation}
\label{eq:hv_model}
{z}^h_{i+1}={A}{z}^h_{i}+{B}{u}_i^h,
\end{equation}
where $A$ and $B$ are the state and control matrices, respectively; the state vector ${z}^h_i = [{{p}^h_i}^T~ {{v}^h_i}^T]^T \in \E^4$ is composed of position ${p}^h_i = [p^h_{x,i}~p^h_{y,i}]^T$ and velocity ${v}^h_i = [v^h_{x,i}~ v^h_{y,i}]^T$, with control input vector ${u}^h_i = [a^h_{x,i}~ a^h_{y,i}]^T \in \mathcal{U}^h$ comprising longitudinal and lateral accelerations. 

% To account for perception inaccuracy, we model the HV state measurement at time \(t_i\) as
% \begin{equation}
% \tilde{z}^h_i = z^h_i + \omega^h_i, \quad \|\omega^h_i\|_{{\Sigma_{\omega}^{h}}^{-1}} \leq 1.
% \end{equation}
% where $\Sigma_{\omega}^h = \sigma^2 I_4$ defines the noise bound.

To account for perception inaccuracy, we model the HV state measurement at time \(t_i\) as
\begin{equation}
\tilde z_i^h = z_i^h + \omega_i^h, \quad \|\omega^h_i\|_{{\Sigma_{\omega}^{h}}^{-1}} \leq 1,
\label{eq:measurement_noise}
\end{equation}
where the noise $\omega_i^h$ is bounded with known, positive-definite covariance $\Sigma_{\omega}^h$.

For the \(h\)-th HV, the control capability is defined by two distinct sets: \textbf{admissible control set} \(\mathcal{U}^h\), which encompasses all physically feasible control inputs under worst-case constraints, and \textbf{control-intent set} \(\hat{\mathcal{U}}^h \subseteq \mathcal{U}^h\), encompassing the control inputs the HV is likely to execute within a finite horizon. The admissible set \(\mathcal{U}^h\) is assumed known, while the intended set \(\hat{\mathcal{U}}^h\) remains unknown \textit{a priori} and should be estimated online through iterative adjustments based on observed behaviors.

To enable smooth trajectory optimization, the state vector \(x_i\in \mathcal{X}\) is augmented to include control inputs and their derivatives, defined as:
\begin{equation} \vspace{-1mm}
\notag
x_i = \left[p_{x,i}\quad p_{y,i}\quad\theta_{i}\quad\dot{\theta}_i\quad v_i\quad a_{x,i}\quad a_{y,i}\quad j_{x,i}\quad j_{y,i} \right]^T.
\vspace{-0.5mm} \end{equation}
Here, $p_{x,i}$ and $p_{y,i}$ denote the longitudinal and lateral positions in the global frame, respectively; $v_i$ is the speed; and $\theta_i$ is the heading angle. Additionally, $a_{x,i}$ and $a_{y,i}$ are the longitudinal and lateral accelerations, while $j_{x,i}$ and $j_{y,i}$ represent the longitudinal and lateral jerks.
Using the yaw rate $\dot{\theta}_i$ and acceleration $a_i$ as control input variables, the EV is modeled as a Dubins car with nonholonomic constraints:
    \begin{eqnarray}
        \begin{aligned}
          \dot{p}_{x,i} -& v_i \cos (\theta_i) = 0,\\
           \dot{p}_{y,i}- &v_i\sin (\theta_i) = 0.
        \end{aligned}
    \label{eq:nonholonomic_cons}
    \end{eqnarray}
Based on these constraints, we can further derive the following constraint of $\theta_i$ and closed-form solution of $v_i$~\cite{chen2023interactive}:
\begin{equation} \vspace{-1mm}
        \theta_i - \arctan (\dot{p}_{y,i}/{\dot{p}_{x,i}}) = 0,
        \label{eq:polar_nonholonomic_cons} 
\end{equation} 
\vspace{-4mm}
\begin{equation} \vspace{-1mm}
        v_i = \sqrt{(\dot{p}_{x,i} )^2 +  (\dot{p}_{y,i} )^2}.
        \label{eq:polar_vel} 
        \vspace{-1mm}
\end{equation}

\subsection{Problem Statement}
\label{subsec:problem} 
This study investigates trajectory optimization for autonomous vehicles in safety-critical environments where $M$ HVs exhibit uncertain behaviors compounded by perception inaccuracies. Complex driving patterns influenced by diverse driving styles and traffic conditions make their trajectories difficult to predict. To robustly handle the uncertainties, we aim to generate collision-free trajectories by avoiding the reachable regions of HVs, characterized via FRSs over a given time horizon.

Conventional FRSs derived from admissible control sets $\mathcal{U}^{h}$ tend to be overly conservative, as they account for worst-case actuator limitations. Conversely, using fixed bounds of control inputs may lead to either excessive caution or dangerous underestimation. In reality, the near-future behaviors of an HV are better characterized by the empirical bounds of historical trajectory data, which captures both driving style and traffic conditions \cite{zhou2025robusta}. 

% accurate prediction of HVs remains challenging, while assuming fixed FRSs fails to adapt to real-time HV behavior variations, leading to risk underestimation or overly conservative. Thus, any viable solution must not only learn and adapt but also consider behavioral uncertainties in the trajectory optimization for safety assurance. % decouple safety assurance from prediction errors
Accordingly, we aim to develop a computationally efficient trajectory optimization framework that enables the EV to accommodate HV uncertainties while ensuring the existence of safe trajectories. Given the EV and HV models in Section \ref{subsec:modeling}, the problem is decomposed into three key subproblems: 

(i) efficiently updating the FRS of HVs \eqref{eq:hv_model} by online learning the unknown control-intent sets $\hat{\mathcal{U}}$; 

(ii) formulating barrier constraints using the FRS to enable safe and computationally efficient trajectory optimization; 

(iii) ensuring the feasibility of safe trajectories while balancing safety, efficiency, and passenger comfort.

\section{FRS-Based Safety Barrier} 
\label{sec:Learning_Reachable_Set} 
This section introduces the FRS-based discrete barrier function method, addressing the subproblems (i) and (ii) in Section \ref{subsec:problem}. Section \ref{subsec:Online_Learning} designs the event trigger mechanism for online learning of HV control-intent sets and FRS propagation. Section \ref{subsec:Reachable_Barrier} then formulates the FRS-based safety barrier constraint.
\vspace{-2mm}

\subsection{FRS Prediction}
\label{subsec:Online_Learning} 

The FRS rigorously characterizes potential future states for an HV, defined as follows:

\begin{definition}[Forward Reachable Set]
The FRS of the $h$-th HV \eqref{eq:hv_model} at the predicted time step $k$, $k \geq 1$, from an initial state vector $z_i$ is defined as:
\begin{equation}
\mathcal{R}_{k|i}^h = \left\{ A^k z_i^h + \sum_{j=0}^{k-1} A^{k-1-j} B u_j^h \,\Bigg|\,\forall u_j^h \in \hat{\mathcal{U}}^h,  j \in \mathcal{I}_0^{k-1}\right\}
\end{equation}
where the state matrix $A$ and control matrix $B$ are defined in \eqref{eq:hv_model}, and $\hat{\mathcal{U}}^h$ denotes the control-intent set of the $h$-th HV.
\end{definition}

Following \cite{zhou2025robustb}, we employ ellipsoidal representations for both the control-intent set and the associated FRS. 
% \textcolor{blue}{This ellipsoidal representation supports efficient online learning while rendering FRS propagation tractable. It also facilitates the construction of safety barrier constraints, improving the computational efficiency for trajectory optimization, as elaborated in Section \ref{subsec:admm}.}
This ellipsoidal representation supports efficient online learning while rendering FRS propagation tractable. Crucially, it facilitates a co-design with the optimization framework, as the resulting safety barrier constraints preserve the biconvex structure essential for the efficient ADMM-based solution elaborated in Section \ref{subsec:admm}.

% \textcolor{blue}{Ellipsoids are chosen for their computational tractability in recursive FRS propagation and their compatibility with our optimization framework. The closed-form Minkowski sum approximation allows efficient multi-step prediction, while the quadratic form preserves the biconvex structure required for efficient ADMM decomposition.}

\subsubsection{Event-Triggered Online Learning of Control-Intent Set}
The control-intent set is initialized according to the EV’s prior knowledge or assumptions on the HV,  and subsequently refined through online learning.
% This initial estimate is then adaptively refined through online updates using historical trajectory data. An event-triggered mechanism is designed to adapt the control-intent set to evolving driving behaviors and traffic conditions, enabling adaptive adjustment and rapid response to unexpected maneuvers.

Specifically, at each time step $t_i$, we obtain the control input vector $u_{i-1}^h$ from the observed state vectors. Then, a control input dataset $\mathcal{C}_i^h=\mathcal{C}_{i-1}^h \cup \{u_{i-1}^h\}$ is maintained for the $h$-th HV, initialized with $\mathcal{C}_0^h$, a user-specified estimate of potential HV control inputs. 

From an estimation perspective, the elements in $\mathcal{C}_i^h$ can be viewed as the samples drawn from an unknown distribution of $\hat{\mathcal{U}}_i^h$. It is over-approximated using a minimum-volume enclosing ellipsoid of $\mathcal{C}_i^h$, denoted as:
\begin{equation}
    \hat{\mathcal{U}}_i^h = \{ u \in \mathbb{R}^2 \mid (u - \mu_{u,i}^h)^T (\Sigma_{u,i}^h)^{-1} (u - \mu_{u,i}^h) \leq 1 \},
    \label{eq:ellipsoid_set}
\end{equation}
with an equivalent form $\|P_{u,i}^h u + q_{u,i}^h\|_2 \leq 1$. The parameters $P_{u,i}^h = (\Sigma_{u,i}^h)^{-1/2}$ and $q_{u,i}^h = -P_{u,i}^h \mu_{u,i}^h$ are obtained by solving the following semi-definite programming (SDP) problem \cite{boyd2004convex}: % to find the Löwner-John ellipsoid
\begin{equation}
\label{eq:full_sdp}
\begin{aligned}
\{P_{u,i}^h, q_{u,i}^h\} &= \underset{\{P,\,q\}}{\text{argmin}} \quad  \log \det(P^{-1}) \\
\text{s.t.} \quad & \|Pu + q\|_2^2 \leq 1, \quad \forall u \in \mathcal{C}_i^h, \\
& P \in \mathbb{S}_{++}^2.
\end{aligned}
\end{equation}
Since ellipsoids are convex, $\hat{\mathcal{U}}_i^h$ naturally contains the convex hull of $\mathcal{C}_i^h$. This ensures that even if the observed control behaviors exhibit non-convex patterns, the learned ellipsoid serves as a valid superset that conservatively envelopes all observed maneuvers.

% \textcolor{blue}{The minimum-volume enclosing ellipsoid computed in Eq.~(7) provides an outer approximation of the convex hull of the dataset $\mathcal{C}_i^h$, guaranteeing $\mathcal{C}_i^h \subseteq \hat{\mathcal{U}}_i^h$ for any data scatter. This ensures safety by conservatively containing all observed control inputs. If the underlying data distribution is multimodal or non-convex, the ellipsoid will enlarge to cover all modes, and the subsequent online learning tightens it progressively, balancing robustness and conservatism.}

For large datasets, Khachiyan's algorithm \cite{todd2007khachiyan} provides efficient computation via first-order approximation. 
Nevertheless, continuously updating $\mathcal{C}_i^h$ with every incoming sample introduces two critical issues. 
First, the increasing dataset size imposes an unnecessary computational that can compromise real-time efficiency. 
Second, and more critically, frequent variations in the learned set geometry can cause oscillations in the feasibility domain, leading to control chattering and degraded ride comfort.
To mitigate these issues, we design the following condition to trigger the updates.

\textbf{Triggering Condition}:\quad The updates occur when
\begin{equation}
t_{m^\prime} = \inf_{i\delta T > t_m} i\delta T
\label{eq:event_trigger_time}
\end{equation}
such that
\begin{equation}
\|P^h_{u, m}u_{i-1}^h + q_{u, m}^h\|_2^2 \geq 1,
\label{eq:trigger_condition}
\end{equation}
where $m,\, m^\prime \in \mathbb{N}$ denote successive update time instants.
% \item \textbf{Event 2}: When exceeding a predefined time threshold $N_c$, prompt a full update based on the initial and recent data $\mathcal{C}_i^h =  \mathcal{C}_{0}^h\cup\bigcup_{j=i-1-N_c}^{i-1} \{u_{j}^h\}$ with a fixed size $N_c$. 
Intuitively, when a new control input sample $u_{i-1}^h$ falls outside the current estimated set $\hat{\mathcal{U}}_m^h$, trigger an incremental update such that  $\mathcal{C}_i^h \subseteq \hat{\mathcal{U}}_{m}^h\cup \{u_{i-1}\}\subseteq \hat{\mathcal{U}}_{m^\prime}^h$.

Once the triggering condition \eqref{eq:trigger_condition} is satisfied, the parameters $ P_{u,m}^h $ and $ q_{u,m}^h $ can be incrementally obtained by solving an SDP problem \eqref{eq:adaptive_ellipsoid_update} with fixed computational complexity, as formalized in the following lemma:

\begin{lemma}\cite{zhou2025robustb}
Given the current ellipsoid $\hat{\mathcal{U}}_{m}^h$ and a new sample $u_{i-1}^h$, the updated parameters $\{P_{u,m^\prime}^h, q_{u,m^\prime}^h\}$ of the  minimum-volume enclosing ellipsoid $\hat{\mathcal{U}}_{m^\prime}^h$ are computed to satisfy $\hat{\mathcal{U}}_{m}^h\cup \{u_{i-1}\}\subseteq \hat{\mathcal{U}}_{m^\prime}^h$ via:
\begin{align}
&\underset{S, c, \lambda_j}{\operatorname*{min}} \quad \log \det(S^{-1}) \nonumber \\
&\operatorname*{s.t.} \quad 
\begin{bmatrix}
S - \lambda_j H_j & c - \lambda_j \zeta_j & 0 \\
(c - \lambda_j \zeta_j)^T & -1 - \lambda_j n_j & c^T \\
0 & c & -S
\end{bmatrix} \preceq 0, \quad j \in \{1,2\},
\label{eq:adaptive_ellipsoid_update}
\end{align}
where $H_1 = (P_{m}^h)^T P_{m}^h$, $\zeta_1 = (P_{m}^h)^T q_{m}^h$, $n_1 = \|q_{m}^h\|_2^2 - 1$; and $H_2 = I_2$, $\zeta_2 = -u_{m}^h$, $n_2 = \|u_{m}^h\|_2^2 - \epsilon$. Here, $I_2\in\mathbb{R}^{2\times2}$ is the identity matrix, and $\epsilon>0$. The optimal solution yields the updated ellipsoid parameters: 
\begin{align}
\label{eq:new_params}
P_{u,m^\prime}^h = (S^*)^{-1/2}, \quad q_{u,m^\prime}^h = (S^*)^{1/2}c^*.
\end{align}
\end{lemma}

% \begin{corollary}
% \label{coroll:update_time}
% At $t_i\in[t_m, t_{m^\prime})$, it holds $u^h_i \in \hat{\mathcal{U}}^h_m$ for $h\in\mathcal{I}_0^{M-1}$.
% \end{corollary}
\begin{corollary}
\label{coroll:update_time}
For all $t_i \in [t_m, t_{m'})$, the control input vector satisfies $u_i^h \in \hat{\mathcal{U}}_m^h$.
\end{corollary}

\begin{remark}
\label{remark:incremental update}
The online learning mechanism \eqref{eq:event_trigger_time}-\eqref{eq:new_params} ensures compatibility with worst-case analysis by requiring that the initial control dataset $\mathcal{C}_0^h$ includes the worst-case control inputs for the $h$-th HV. 
\end{remark}

% \begin{remark}
% \label{remark:incremental update}
% \textcolor{blue}{The online learning mechanism \eqref{eq:event_trigger_time}-\eqref{eq:new_params} intentionally over-approximates non-convex datasets to ensure implicit coverage of intermediate maneuvers and preserve convexity for efficient solving.}
% Furthermore, requiring the initial dataset $\mathcal{C}_0^h$ to include worst-case limits ensures compatibility with worst-case analysis.
% \end{remark}

\subsubsection{FRS Propagation}
Using the learned control-intent set $\hat{\mathcal{U}}_{m}^h$, we propagate the FRSs over the prediction horizon $N$, starting at time step $t_i\geq t_m$.

% The HV state $z_i^h$ at $t_i$ is estimated via Kalman filtering. It gives an ellipsoidal initial FRS as:
% \begin{align}
%     \mathcal{R}_{0|i}^{h,z} = \{ z \in \mathbb{R}^4 \mid (z - \tilde{z}_i^h)^\top (\Sigma_{z,0|i}^h)^{-1} (z - \tilde{z}_i^h) \leq 1 \},
%     \label{eq:init_reachable_set}
% \end{align}
% where $\tilde{z}_i^h$ is the estimated mean state, and $\Sigma_{z,0|i}^h=\hat \Sigma_{i}^h/\kappa$ is computed using estimation covariance $\hat \Sigma_{i}^h$ with confidence level parameter \(\kappa\). % (e.g., $\kappa=9.49$ for 95\% confidence).}

Considering perception noise, the initial FRS is given by:
\begin{align}
    \mathcal{R}_{0|i}^{h,z} = \{ z \in \mathbb{R}^4 \mid (z - \tilde{z}_i^h)^\top (\Sigma_{z,0|i}^h)^{-1} (z - \tilde{z}_i^h) \leq 1 \},
    \label{eq:init_reachable_set}
\end{align}
where $\Sigma_{z,0|i}^h=\Sigma_{\omega}^h$.

Then, the FRS $\mathcal{R}_{k|i}^{h,z}$ at predicted time step $k$ starting from $t_{i}$ is computed recursively:
\begin{align}
    \mathcal{R}_{k|i}^{h,z} &= A \mathcal{R}_{k-1|i}^{h,z} \oplus B \hat{\mathcal{U}}_m^h, \quad k \in \mathcal{I}_1^N. 
    \label{eq:reachable_recursion}
\end{align}
 
Although $\mathcal{R}_{0|i}^{h,z} $ and $ \hat{\mathcal{U}}_m^h $ are ellipsoids,  their Minkowski sum $\mathcal{R}_{k|i}^{h,z} $ is typically not an ellipsoid \cite{berkenkamp2017safe}. We therefore construct an efficient ellipsoidal over-approximation:
$$
\hat{\mathcal{R}}_{k|i}^{h,z} = \left\{ z \in \mathbb{R}^4 \, | \, (z - \mu_{z,k|i}^{h})^{T}(\hat \Sigma_{z,k|i}^{h})^{-1}(z - \mu_{z,k|i}^{h}) \leq 1 \right\},
$$
with parameters $(\mu_{z,k|i}^{h}, \hat{\Sigma}_{z,k|i}^{h})$ given by:
\begin{align}
    \mu_{z,k|i}^{h} &= A \mu_{z,k-1|i}^{h} + B \mu_{u,m}^h, \label{eq:mean_propagation} \\
    \hat{\Sigma}_{z,k|i}^{h} &= \kappa ( {\tilde{\Sigma}_{z,k-1|i}^{h}}/{\rho_{z,k-1|i}} + {\tilde{\Sigma}_{u,m}^h}/{\rho_{u,m}} ), \label{eq:covariance_scaling}
\end{align}
where $\tilde{\Sigma}_{z,k-1|i}^{h} = A \hat\Sigma_{z,k-1|i}^{h} A^T$; $\tilde{\Sigma}_{u,m}^h = B \Sigma_{u,m}^h B^T + \epsilon_2 I_4$; and directional scaling factors $\rho_{z,k-1|i} = \sqrt{l^T \hat{\Sigma}_{z,k-1|i}^{h} l}$, $\rho_{u,m} = \sqrt{l^T \tilde{\Sigma}_{u,m}^h l}$, $\kappa = \rho_{z,k-1|i} + \rho_{u,m}$ are computed along direction $l$.

% \begin{remark}[On state constraints in FRS propagation]
The FRS propagation in (6) and (14) omits explicit state constraints (e.g., velocity bounds) to maintain a recursive ellipsoidal form and computational tractability.  % With ensured containment relationship, the associated conservatism of ellipsoidal over-approximation is actively managed by the online learned $\hat{\mathcal{U}}^h$ and , enabling a practical balance between robustness and efficiency, as validated in Sec.~V.
% \end{remark}

\subsubsection{Occupancy Prediction} 
Based on the FRS prediction, the positional occupancy $\mathcal{R}_{k|i}^{h} = \text{Proj}_{\text{position}}(\hat{\mathcal{R}}_{k|i}^{h,z})$ is obtained via projection:
\begin{equation}
    {\mathcal{R}}_{k|i}^{h} = \left\{ p \in \mathbb{R}^2  \, | \,  (p - \mu_{p,k|i}^{h})^{T}(\hat\Sigma_{p,k|i}^{h})^{-1}(p - \mu_{p,k|i}^{h}) \leq 1 \right\}.
    \label{eq:occupancy_projection}
\end{equation}
However, considering only the HV center while neglecting the geometry is insufficient for safety assurance. Therefore, we employ an ellipsoidal geometry set to explicitly account for the full shape characteristics of both the HV and EV, which is represented by $\mathcal{D}_{safe} = \left\{ p \in \mathbb{R}^2 \, |\, p^{T}(\Sigma^h_{safe})^{-1}p \leq 1 \right\}$. 
% Its major and minor axis lengths are denoted as $l^h_{x}$ and $l^h_{y}$, respectively.
Its lengths of major and minor axes are denoted as $l^h_{x}$ and $l^h_{y}$, respectively.

Thus, the complete HV FRS occupancy can be derived as an external ellipsoidal approximation of $\mathcal{R}_{k|i}^{h}\oplus\mathcal{D}_{safe}$, denoted by:
\begin{equation}
    \hat{\mathcal{R}}_{k|i}^{h} = \left\{ p \in \mathbb{R}^2 \,|\, (p - \mu_{p,k|i}^{h})^{T}(\hat\Sigma_{p,k|i}^{h})^{-1}(p - \mu_{p,k|i}^{h}) \leq 1 \right\},
    \label{eq:actual_occupancy_projection}
\end{equation}
where $\mu_{p,k|i}^{h} = [\mu_{x,k|i}^{h}~\mu_{y,k|i}^{h}]^T$ denotes the center of the positional occupancy $\hat{\mathcal{R}}_{k|i}^{h}$. 
The ellipsoidal approximation can be computed similarly to \eqref{eq:mean_propagation} and \eqref{eq:covariance_scaling}. We denote the lengths of major and minor axes of the ellipsoid $\hat{\mathcal{R}}_{k|i}^{h}$ as $l^h_{x,k|i}$ and $l^h_{y,k|i}$, respectively.

% \begin{remark}
% \label{remark:perception_inaccuracies}  
% Perception inaccuracies are quantified by the estimation confidence $\Sigma_{z,0|i}^h$ and integrated into the initial FRS, ensuring the predicted occupancy reflects sensing uncertainties for robust collision checking. The estimation confidence can be set to adequate high to meet the requirement of assumption.
% \end{remark}  

\begin{remark}
\label{remark:perception_inaccuracies}
In practice, the bound of noise in \eqref{eq:measurement_noise} can be estimated via a variety of filtering techniques (e.g., Kalman filter). The estimation covariance quantifies perception uncertainties, which can be integrated into the initial FRS. By conservatively tuning the filter parameters and confidence level, the resulting FRS robustly contains the true state, enabling reliable collision checking.
\end{remark}

\begin{remark}
\label{remark:implicit_robustness}  
The initial control-intent set encodes the EV's a priori knowledge of HV behaviors and implicitly determines the basic conservatism of the planning framework. The system adapts dynamically through event-triggered learning, refining these sets based on observed behaviors, thus improving prediction accuracy over time. %The future behaviors of HVs are more likely to fall within the corresponding control-intent sets, as more behavioral evidence is acquired.
\end{remark}  

% \begin{remark}
% \label{remark:implicit_robustness}
% The initial control-intent set represents the EV's a priori knowledge of HV behavior and implicitly dictates the conservatism of the planning framework. The system adapts dynamically through event-triggered learning, refining these sets based on observed behaviors, thus improving prediction accuracy over time.
% \end{remark}

Let $\hat{\mathcal{O}}^h_{k|i} = z^h_{p,{k|i}} \oplus \mathcal{D}_{\text{safe}}$ denote the predicted positional occupancy at predicted time step $k$ from $t_i$, with $z^{h}_{p,{k|i}} = [z^{h}_{x,{k|i}}\, z^{h}_{y,{k|i}}]^T$ denoting the position vector, while ${\mathcal{O}}^h_{k|i}$ represents the actual positional occupancy at time instant $t_{i+k}$.

\begin{lemma}
\label{lemma:merged_reach_containment}
Consider control input vectors of the $h$-th HV during planning horizon $N$ satisfying
$
u_{k|i}^h \in \hat{\mathcal{U}}_m^h,  \forall k \in \mathcal{I}_0^{N-1},   
$ for all $t_i \in [t_m, t_{m'-N})$,
then the following containment relations hold: 
\begin{equation*}
\mathcal{O}^h_{k|i} \subseteq \hat{\mathcal{R}}_{k|i}^h,  \forall k \in \mathcal{I}_0^{N} \quad \text{and} \quad \hat{\mathcal{R}}_{k|i+1}^h \subseteq \hat{\mathcal{R}}_{k+1|i}^h ,  \forall k \in \mathcal{I}_0^{N-1}.
\end{equation*}
\begin{proof}
The initial FRS $\mathcal{R}_{0|i}^{h,z}$ in \eqref{eq:init_reachable_set} is an ellipsoid centered at the estimated HV state vector $\hat z_i^h$ with covariance $\Sigma_{z,0}^h$, which ensures $z_{0|i}^{h}=z_i^h \in \mathcal{R}_{0|i}^{h,z}$. Since $u_{k|i}^h \in \hat{\mathcal{U}}_m^h$ for all $k \in \mathcal{I}_0^{N-1}$, the recursive propagation in \eqref{eq:reachable_recursion} ensure containment of the true state trajectory within the FRSs as well as their over-approximations: $z_{k|i}^{h}\in\mathcal{R}_{k|i}^{h,z}\subseteq \hat{\mathcal{R}}_{k|i}^{h,z}$.

Projecting $\hat{\mathcal{R}}_{k|i}^{h,z}$ to positional space via \eqref{eq:occupancy_projection} preserves this containment, yielding $z_{p,k|i}^{h} \in \mathcal{R}_{k|i}^{h}$ for all $k \in \mathcal{I}_0^N$.

Incorporating the vehicle geometry $\mathcal{D}_{safe}$ through Minkowski sum and its ellipsoidal approximation ensures $z_{p,k|i}^{h} \oplus \mathcal{D}_{safe} \subseteq \mathcal{R}_{k|i}^{h} \oplus \mathcal{D}_{safe}$. Following the over-approximation property of the external ellipsoid, $\mathcal{O}^h_{k|i} \subseteq \hat{\mathcal{R}}_{k|i}^{h}$  holds for all $k \in \mathcal{I}_0^{N}$.

With the initial condition $z^h_{p,0|i+1}\in \hat{\mathcal{R}}_{0|i+1}^h\subseteq \hat{\mathcal{R}}_{1|i}^h$, the ellipsoidal approximation preserves the containment relationship $\hat{\mathcal{R}}_{k|i+1}^h \subseteq \hat{\mathcal{R}}_{k+1|i}^h, \forall k \in \mathcal{I}_0^{N-1}$. 
\end{proof}
\end{lemma}

\subsection{FRS-Based Safety Barrier}
\label{subsec:Reachable_Barrier} 

The safe set can be  characterized by a sublevel set of a continuously differentiable function $\mathscr{B}$:
\begin{equation}
\mathcal{F}(\mathcal{O}^{h}) =\{x_i\in\mathcal{X} \mid \mathscr{B}(x_i, \mathcal{O}^{h})\geq 0\} \ \label{eq:safe_set} 
\end{equation}

\begin{definition}[Forward Invariance]\label{def:forward_invariance}
The set $\mathcal{F}(\mathcal{O}^{h})$ is \emph{forward invariant} if for any initial state vector $x_0 \in \mathcal{F}(\mathcal{O}^{h})$, the condition $\mathscr{B}(x_i, \mathcal{O}^{h}) \geq 0$ holds for all $i\geq0$.
\end{definition}

\begin{lemma}\cite[Proposition 4]{agrawal2017discrete}
\label{thm:forward_invariance} 
The safe set $\mathcal{F}(\mathcal{O}^{h})$ is forward invariant, if $\mathscr{B}$ with initial condition $\mathscr{B}(x_{0}, \mathcal{O}^{h})\geq0$ satisfies:
\begin{equation} 
\label{eq:spatiotemporal_barrier_cons}
\Delta  \mathscr{B}(x_i, \mathcal{O}^{h}) + \alpha \mathscr{B}(x_{i}, \mathcal{O}^{h}) \geq0 ,
 \end{equation} 
 where $\Delta  \mathscr{B}(x_{i}, \mathcal{O}^{h})=\mathscr{B}(x_{i+1}, \mathcal{O}^{h})-  \mathscr{B}(x_{i}, \mathcal{O}^{h})$, $0< \alpha  \leq 1$, and $\mathscr{B}$ is said to be a discrete-time barrier function.
\end{lemma}

Following the methodology in \cite{zheng2024barrier}, we employ polar coordinate transformations to formulate the nominal \textbf{deterministic safety barrier constraints} for collision avoidance as: 
\begin{eqnarray}
\left\{
    \begin{aligned}
       p_{x,{k|i}}& = z^{h}_{x,{k|i}} + l^{h}_{x} d^{h}_{k|i} \cos(\omega^{h}_{k|i}), \\  
      p_{y,{k|i}}& = z^{h}_{y,{k|i}} + l^{h}_{y} d^{h}_{k|i} \sin(\omega^{h}_{k|i}), \\  
   \Delta  \mathscr{B}&{(x_{k|i}, \hat{\mathcal{O}}^{h}_{k|i}) + \alpha  \mathscr{B}(x_{k|i}, \hat{\mathcal{O}}^{h}_{k|i})  \geq 0},\, k\in\mathcal{I}_0^{N-1}.   
    \end{aligned}
\right.
\label{eq:polar_prediction_safety}
\end{eqnarray} 
% where the relative orientation angle between the EV and the positional occupancy of HV is given by: 
Here, the parameters $l^{h}_{x}$ and $l^{h}_{y}$ represent the semi-axes of the ellipsoidal safety boundary, and the barrier function $\mathscr{B}$ is explicitly defined as:
\begin{equation}
\mathscr{B}(x_{k|i}, \hat{\mathcal{O}}_{k|i}^{h})=d_{k|i}^{h}(x_{k|i}, \hat{\mathcal{O}}_{k|i}^{h})-1,
\end{equation}
where $d_{k|i}^{h}$ is the Mahalanobis distance from $x_{k|i}$ to the ellipsoid $\hat{\mathcal{O}}_{k|i}^{h}$. The condition $\mathscr{B} \ge 0$ ensures that the EV state remains outside the ellipsoidal obstacle region. Intuitively, when a trajectory point $[p_{x,{k|i}}\, p_{y,{k|i}}]^T$ of the EV penetrates the obstacle region centered at $[z^{h}_{x,{k|i}}\, z^{h}_{y,{k|i}}]^T$, it is pushed outward along an optimal repulsion direction $\omega^{h}_{k|i}$:
\begin{equation}
\label{eq:obtain_omega}
\omega^{h}_{k|i} = \arctan\left(\frac{l^{h}_{x}(p_{y,{k|i}} - z^{h}_{y,{k|i}})}{l^{h}_{y}(p_{x,{k|i}} - z^{h}_{x,{k|i}})}\right).
\end{equation} 
Note that the constraint \eqref{eq:polar_prediction_safety} characterizes the collision avoidance by maintaining a minimum safety margin throughout the planning horizon, with an equivalent polar coordinate formulation:
   \begin{equation} \vspace{-1mm} 
\label{eq:barrier_cons_polar}
 d^{h}_{k+1|i} -1  - (1- \alpha)  ( d^{h}_{k|i} -1) \geq 0. 
  \vspace{-1mm} \end{equation}  

\begin{figure}[tb]
\begin{center}
\includegraphics[width=.75\columnwidth]{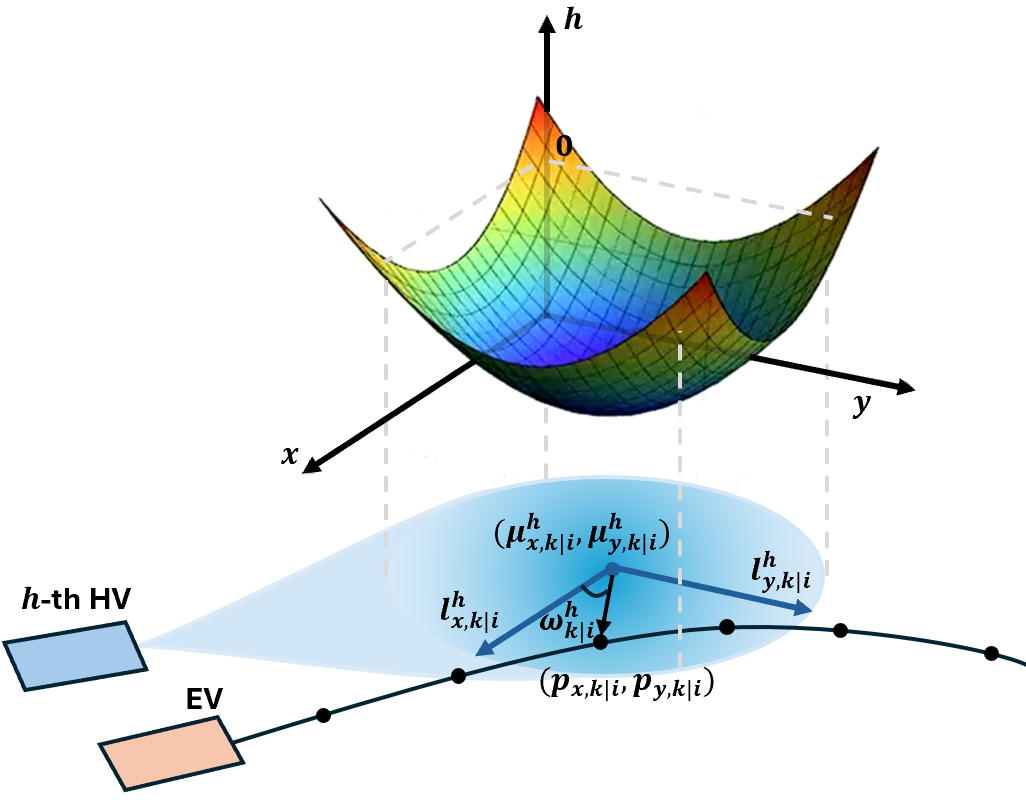}    \vspace{-3mm}
\caption{  
Illustration of the FRS-based safety barrier. At predicted time step $k$ starting from $ t_{i} $, the trajectory point $p_{k|i}$ of the EV is pushed away from the FRS positional occupancy prediction of the HV, along the direction $ \omega_{k|i}^{h} $. The safe set $ F(\hat{\mathcal{R}}_{k|i}^{h})$ (outside the ellipsoid) is encoded using discrete barrier function $\mathscr{B} \geq 0$.
} \vspace{-8mm}
\label{fig:barrier}
\end{center}
\end{figure} 

However, since accurate trajectory prediction of HVs are typically unavailable, we utilize the obtained ellipsoid FRS occupancy $\hat{\mathcal{R}}_{k|i}^{h}$ to construct the \textbf{FRS-based safety barrier constraints} as shown in Fig. \ref{fig:barrier}, which is derived as:
\begin{eqnarray}
\left\{
    \begin{aligned}
       p_{x,{k|i}}& = \mu^{h}_{x,{k|i}} + l^{h}_{x,{k|i}} d^{h}_{k|i} \cos(\omega^{h}_{k|i}), \\  
      p_{y,{k|i}}& = \mu^{h}_{y,{k|i}} + l^{h}_{y,{k|i}} d^{h}_{k|i} \sin(\omega^{h}_{k|i}), \\  
   \Delta  \mathscr{B}&{(x_{k|i}, \hat{\mathcal{R}}^{h}_{k|i}) + \alpha  \mathscr{B}(x_{k|i}, \hat{\mathcal{R}}^{h}_{k|i})  \geq 0},\, k\in\mathcal{I}_0^{N-1}.   
    \end{aligned}
\right.
\label{eq:polar_reachable_safety}
\end{eqnarray}

The barrier function introduces a fundamental duality principle that a larger obstacle region induces a smaller safe set.

\begin{lemma}
\label{lem:safe-set-duality}
Let $\hat{\mathcal{R}}_A$ and $\hat{\mathcal{R}}_B$ be two ellipsoidal sets in $\mathbb{R}^2$ with centers $z_A, z_B$ and positive definite shape matrices $Q_A, Q_B$, respectively. The following containment relationship holds: 
\[
\hat{\mathcal{R}}_A \subseteq \hat{\mathcal{R}}_B  \Longrightarrow  \mathcal{F}\left(\hat{\mathcal{R}}_B\right)\subseteq\mathcal{F}\left(\hat{\mathcal{R}}_A\right),
\]
\end{lemma}
\begin{proof}
For any $x \in \mathcal{F}(\hat{\mathcal{R}}_B)$, the Mahalanobis distance satisfies $
d(x, \hat{\mathcal{R}}_B) = \|x - z_B\|_{Q_B^{-1}} \geq 1$.
The containment  $\hat{\mathcal{R}}_A \subseteq \hat{\mathcal{R}}_B$ implies $
\{x \mid \|x - z_A\|_{Q_A^{-1}} \leq 1\} \subseteq \{x \mid \|x - z_B\|_{Q_B^{-1}} \leq 1\}$.
Consequently, for $x \notin \text{Int}(\hat{\mathcal{R}}_B)$, we have $\|x-z_A\|_{Q_A^{-1}}\geq 1$. Thus, $d(x, \hat{\mathcal{R}}_A) - 1\geq0$, with equality only for $x \in \partial\hat{\mathcal{R}}_A$. 
Therefore, it gives $x \in \mathcal{F}(\hat{\mathcal{R}}_A)$, and this completes the proof.
\qedhere
\end{proof}

It can be shown that the $\mathcal{F}(\hat{\mathcal{R}}_{k|i}^{h})$ is a robust safe set of the EV with respect to the $h$-th HV.

\begin{lemma}
\label{lem:robust_forward_invariance}
Let the $h$-th HV satisfies $u^h_{k|i} \in \hat{\mathcal{U}}^h_m$.  %, $t_i \in [t_m, t_{m'-N})$
The EV trajectory with an initial state vector $x_{k|i} \in \mathcal{F}(\hat{\mathcal{R}}_{k|i}^{h})$ can robustly avoid collisions $x_{k+1|i} \in \mathcal{F}(\hat{\mathcal{R}}_{k+1|i}^{h})\subseteq \mathcal{F}(\mathcal{O}^{h}_{k+1|i})$, $\forall k\geq0$ for all realizable trajectories of the HV, if the barrier condition \eqref{eq:polar_reachable_safety} is satisfied.
\end{lemma}
\begin{proof}
Since the initial state vector $x_{k|i} \in \mathcal{F}(\hat{\mathcal{R}}_{k|i}^{h})$, we have $\mathscr{B}(x_{k|i}, \hat{\mathcal{R}}^{h}_{k|i})\geq0$ according to \eqref{eq:safe_set}. From \eqref{eq:spatiotemporal_barrier_cons}, one can derive 
$\mathscr{B}(x_{k+1|i}, \hat{\mathcal{R}}^{h}_{k+1|i})\geq(1-\alpha)\mathscr{B}(x_{k|i}, \hat{\mathcal{R}}^{h}_{k|i})\geq0$. Thus, $x_{k+1|i} \in \mathcal{F}(\hat{\mathcal{R}}_{k+1|i}^{h})$. Since $\mathcal{O}^{h}_{k+1|i} \subseteq \hat{\mathcal{R}}^{h}_{k+1|i}$  holds for control input vectors $u^h_{k|i} \in \hat{\mathcal{U}}^h_m$ by Lemma \ref{lemma:merged_reach_containment}, we have  $\mathcal{F}(\hat{\mathcal{R}}^{h}_{k+1|i}) \subseteq \mathcal{F}(\mathcal{O}^{h}_{k+1|i})$ according to Lemma~\ref{lem:safe-set-duality}. Consequently, $x_{k+1|i} \in \mathcal{F}(\hat{\mathcal{R}}^{h}_{k+1|i}) \subseteq \mathcal{F}(\mathcal{O}^{h}_{k+1|i})$ holds. \qedhere
\end{proof}

\vspace{-2mm}
 \section{Safe Contingency Trajectory Optimization} 
 \label{sec:optimization_problem}
 
In this section, we first introduce the trajectory representation based on B\'ezier curves for smoothness and computational efficiency. Next, we formulate a contingency optimization problem with synergistically coupled FRS-based safety barrier constraints to enforce safety. Finally, leveraging the biconvex structure of the optimization problem, we derive an equivalent reformulation that enables efficient alternating optimization using ADMM.

\vspace{-2mm}
\subsection{Trajectory Parameterization}
% This work employs the B\'ezier curve to parameterize the trajectory of the augmented state $x$ \eqref{eq:extended_state}, capturing the evolution of the EV state. Specifically, three B\'ezier curves are optimized to 
We employ B\'ezier curves to parameterize EV trajectories over a finite duration of $T$. This representation exploits the hodograph property to directly constrain higher-order derivatives, ensuring dynamical feasibility \cite{Farin2002}. A degree-$n$ B\'ezier curve is constructed using $(n+1)$ control points over a normalized interval $\nu \in [0, 1]$:
\begin{equation}
    b(\nu) = \sum_{j=0}^{n} c_{j} b_{j,n}(\nu)=c^T\mathcal{B}_v(\nu),
    \notag
\end{equation}
where $c = [c_{0}\, \dots\, c_{n}]^T \in \mathbb{R}^{n+1}$ are control points, $\mathcal{B}_v(\nu) = [b_{0,n}(\nu)\,\dots\,b_{n,n}(\nu)]^T\in \mathbb{R}^{n+1}$ consists of Bernstein polynomial basis $b_{j,n}(\nu)= \binom{n}{j} \nu^j (1 - \nu)^{n-j}$, and $\nu = (t - t_i)/T\in[0,1]$ is the normalized time. Discrete-time points are obtained as:
\[
b_{k} = c^T\mathcal{B}_{k},
\]
where \(\mathcal{B}_{k}=\mathcal{B}_v(k\delta T)\), and $\delta T=T/N$ is the fixed time step. 

The state vector of EV can then be parameterized as:
\[
x_{k|i}=\mathfrak{p}(c_x, c_y, c_\theta).
\] 
Specifically, three distinct B\'ezier curves are used to represent the longitudinal displacement, lateral displacement and heading angle, which are denoted by $p_{x,k|i}=c_x^T \mathcal{B}_{k}$, $p_{y,{k|i}}=c_y^T\mathcal{B}_{k}$, and $p_{\theta,{k|i}}=c_\theta^T\mathcal{B}_{k}$, respectively, with corresponding control points $c_x$, $c_y$, and $c_\theta$. Higher-order derivatives follow from the hodograph property:
$\dot{\theta}_{k|i}={c_\theta}^T\dot{\mathcal{B}}_{k}$, $v_{k|i}=({c_x}^T\dot{\mathcal{B}}_{k} + {c_y}^T\dot{\mathcal{B}}_{k})^{\frac{1}{2}}$,  
               $a_{x,{k|i}}={c_x}^T\ddot{\mathcal{B}}_{k}$, $a_{y,{k|i}}={c_y}^T\ddot{\mathcal{B}}_{k}$, $j_{x,{k|i}}={c_x}^T\dddot{\mathcal{B}}_{k}$, and $j_{y,{k|i}}={c_y}^T\dddot{\mathcal{B}}_{k}$.

 \vspace{-2mm}
\subsection{Safe Contingency Trajectory Optimization}
\label{subsec:original_problem}
We formulate a trajectory optimization problem through simultaneous optimization of a nominal trajectory and a contingency trajectory constrained by FRS-based safety barriers. For clarity, the following formulations are presented for an arbitrary HV indexed by $h \in \mathcal{I}_0^{M-1}$, with all $h$-indexed quantities defined per-HV and naturally extendable to all $M$ HVs.

While more accurate prediction methods exist \cite{huang2022survey}, we intentionally employ a constant-velocity model to explicitly demonstrate the robustness to prediction errors. 
Notably, our framework works with any HV prediction method as long as the predicted states meet this basic containment condition:

\begin{assumption}
\label{assum:predict_and_reachable}
The predicted HV position occupancy $\hat{\mathcal{O}}^{h}_{k|i}$ is contained within the predicted FRS position occupancy:
\[
\hat{\mathcal{O}}^{h}_{k|i} \in \hat{\mathcal{R}}^{h}_{k|i},\quad \forall k \in \mathcal{I}_0^N.
\]
\end{assumption}
This assumption is enforceable because the HV prediction method can be systematically designed to ensure containment through proper control constraint characterization. It is trivially satisfied under constant-velocity prediction due to the conservative nature of FRS propagation.

The optimization problem at $t_i$ is formulated as: 
\begin{subequations}
    \label{problem}\begin{align}
    &\displaystyle\operatorname*{min}_{\substack{
 \{c^r_x,c^r_y,c^r_\theta, c^s_x,c^s_y,c^s_\theta\} 
      }}~~
     \mathcal{J}\label{eq:problem}\\
    \quad\text{s.t.}\quad
    &x^{r}_{0|i}=x(t_i), \quad x^{s}_{0|i}=x(t_i), \label{eq:init_states}\\
    & x^{r}_{k|i} \,\in  \mathcal{X}, \quad \,\, \,\,\,\, \, x^{s}_{k|i} \,\in  \mathcal{X}\label{eq:problem_6}\\ 
    & x^{r}_{N|i} \in  \mathcal{X}_f, \quad \,\, \, \, x^{s}_{N|i} \in  \mathcal{X}_f\label{eq:problem_final}\\ 
    % &x^{s}_{N|i} \in \mathcal{X}^{s}_f, \quad x^{r}_{N|i} \in \mathcal{X}^{r}_f, \label{eq:terminal_states}\\
    & x^{r}_{k+1|i}  = f(x^{r}_{k|i}),  \label{eq:kinematics_r}\\ 
    & x^{s}_{k+1|i}  = f(x^{s}_{k|i}),  \label{eq:kinematics_s}\\ 
    & x^{r}_{k|i}=\mathfrak{p}(c^r_x, c^r_y, c^r_{\theta}), \label{eq:bezier_r}\\
    & x^{s}_{k|i}=\mathfrak{p}(c^s_x, c^s_y, c^s_{\theta}), \label{eq:bezier_s}\\
    % & x^{r}_{k|i} \in  \bigcap_{h=0}^{M-1} \mathcal{F}(\mathcal{O}_{k|i}^{h}), \label{eq:nominal_safety}\\
    % & x^{s}_{k|i} \in  \bigcap_{h=0}^{M-1} \mathcal{F}(\mathcal{R}_{k|i}^{h}), \label{eq:reachable_safety}\\
    % &\text{Constraints \eqref{eq:polar_prediction_safety} hold for } x^{r}_{k|i}, \label{eq:nominal_safety}\\
    % &\text{Constraints \eqref{eq:polar_reachable_safety} hold for } x^{s}_{k|i}, \label{eq:reachable_safety}\\
    & x^{r}_{k|i} \text{ satisfies \eqref{eq:polar_prediction_safety} }, \label{eq:nominal_safety}\\
    & x^{s}_{k|i} \text{ satisfies \eqref{eq:polar_reachable_safety} }, \label{eq:reachable_safety}\\
    & \varphi( x^{r}_{l|i}) = \varphi( x^{s}_{l|i}), \quad  l\in\mathcal{I}_0^{N_s-1}, \label{eq:tie_5}\\ 
    & \forall k\in\mathcal{I}_0^{N-1}. \nonumber \vspace{-0mm}
\end{align}
\end{subequations} 
Here, $x^r_k$ and $x^s_k$ represent the state vectors for the nominal and safe contingency trajectories, respectively, with corresponding control points $c^r_x$, $c^r_y$, $c^r_\theta$, and $c^s_x$, $c^s_y$, $c^s_\theta$. \(x(t_i)\) denotes current state vector of the EV at $t_i$. The function \(f(\cdot)\) implements the kinematic constraint \eqref{eq:nonholonomic_cons}. 

The objective function $\mathcal{J}$ combines the desired lateral position $p_{y,d}$ and longitudinal speed $v_{x,d}$ with trajectory smoothness, expressed as: 
\begin{equation}
\begin{aligned}
    \small
    \mathcal{J}  = &(1-p_s) \sum_{k=0}^{N-1} w^r_x({c^r_x}^T{c^r_x})^2 
    + w^r_y({c^r_y}^T{c^r_y})^2 + w^r_\theta({c^r_\theta}^T{c^r_\theta})^2 \\ 
    & + w^r_{v,d}({c^r_x}^T\dot{\mathcal{B}}_k-v_{x,d})^2+ w^r_{y,d}({c^r_y}^T\mathcal{B}_k-p_{y,d})^2\\
    &+ p_s \sum_{k=0}^{N-1} w^s_x({c^s_x}^T{c^s_x})^2 
    + w^s_y({c^s_y}^T{c^s_y})^2 + w^s_\theta({c^s_\theta}^T{c^s_\theta})^2 \\ 
    & + w^s_{v,d}({c^s_x}^T\dot{\mathcal{B}}_k-v_{x,d})^2+ w^s_{y,d}({c^s_y}^T\mathcal{B}_k-p_{y,d})^2,
    \label{eq:obj_func}
\end{aligned} 
\end{equation}
where the weights $p_s$, $w^r_x$, $w^r_y$, $w^r_\theta$, $w^r_{v,d}$, $w^r_{y,d}$, $w^s_x$, $w^s_y$, $w^s_\theta$, $w^s_{v,d}$, and $w^s_{y,d}$ are all positive constants. 

The two parallel trajectories share the same dynamics inherent in the hodograph property of the  B\'ezier curve \eqref{eq:bezier_r}-\eqref{eq:bezier_s}, initial conditions \eqref{eq:init_states}, terminal state constraints \eqref{eq:problem_final} with terminal state set $\mathcal{X}_f$, and kinematic constraints as stated in \eqref{eq:kinematics_r}-\eqref{eq:kinematics_s}. We tie them together by requiring the initial segments of length $N_s$, $N_s\leq N$, to be consistent \eqref{eq:tie_5}.  For safety assurance, the state vectors \( x^r_k \) over the nominal horizon are required to remain within the safe state set with respect to the HV predictions \( \hat{\mathcal{O}}^{h}_k\), as stated in \eqref{eq:nominal_safety}. To handle behavioral uncertainty, we impose the FRS-based safety barrier \eqref{eq:reachable_safety} as safeguards against possible HV behaviors in the contingency horizon.

The weighting parameter $p_s$ balances nominal planning and contingency handling via FRS-based barrier constraints and can be adjusted according to the desired level of conservatism. The contingency planning formulation in \eqref{problem} minimizes the expected cost over both trajectories. As prediction confidence increases, smaller values of $p_s$ can be used. This approach generates a nominal trajectory based on prediction \( \hat{\mathcal{O}}_k(\cdot) \) while consistently maintaining a safe contingency plan.

\vspace{-4mm}
\subsection{Planning-Horizon Configurations with Formal Guarantees}\label{subsec:horizon_configs}
We analyze two standard planning-horizon configurations used by the proposed planning framework: a \emph{shrinking horizon} and a \emph{receding horizon}. The former captures short, goal-directed maneuvers that must complete within a known mission duration (e.g., mandated lane change to reach an exit, controlled stop into a refuge bay). It also serves as a safety mode should perception temporarily fail. The latter reflects the normal navigation regime in which fresh predictions and FRS updates arrive as time advances.

\paragraph*{\textbf{Shrinking-Horizon Configuration}}
Let $N_{\mathrm{SH}}$ denote the total number of mission steps remaining at $t_0$. At generic time step index $i$ with current time $t_i$, the optimization horizon in~\eqref{eq:problem} is chosen to span the \emph{remaining} mission duration:
\begin{equation}
\label{eq:shrinking_horizon_n}
    N = N_{\mathrm{SH}} - i.
\end{equation}
This configuration is appropriate when the EV must guarantee completion of a finite-time maneuver under bounded uncertainty without relying on future FRS updates. 

We next establish that, under mild conditions, feasibility and safety propagate forward as the horizon shrinks.

\begin{theorem}[Shrinking-Horizon Recursive Feasibility and Safety]\label{thm:shrink_rec_feas}
Suppose: (i) \emph{no HV FRS update events are triggered} over the finite mission interval $[t_0, t_{N_{\mathrm{SH}}}]$; and (ii) Problem~\eqref{problem} with horizon~\eqref{eq:shrinking_horizon_n} is feasible at $t_0$ for B\'ezier curve order $n \geq N_{\mathrm{SH}}-1$ and barrier coefficient $\alpha = 1$.  Then the following hold for every $t_i \le t_{N_{\mathrm{SH}}}$:
\begin{enumerate}
    \item \textbf{Recursive Feasibility:} Problem~\eqref{problem} with horizon~\eqref{eq:shrinking_horizon_n} remains feasible.
    \item \textbf{Open-Loop Safety:} The contingency trajectory satisfies
    $x^s_{k|i} \in \mathcal{F}(\mathcal{O}^h_{k|i})$ for all $k \in \mathcal{I}_0^{N_{\mathrm{SH}}-i}$.
    \item \textbf{Closed-Loop Safety:} The executed trajectory $\{x_i\}_{i=0}^{N_{\mathrm{SH}}}$ satisfies $x_i \in \mathcal{F}(\mathcal{O}^h_i)$ at each time step.
\end{enumerate}
\end{theorem}

\begin{proof}
Let $\{x^{r*}_{0|i},\dots,x^{r*}_{N_{\mathrm{SH}}-i|i}\}$ and $\{x^{s*}_{0|i},\dots,x^{s*}_{N_{\mathrm{SH}}-i|i}\}$ denote the optimal nominal and contingency state sequences obtained from~\eqref{problem} at time $t_i$. Form a \emph{shifted candidate} sequence for the problem at the next time $t_{i+1}$:
\[
    x_{k|i+1} = x^{s*}_{k+1|i}, \quad k=\mathcal{I}_0^{N_{\mathrm{SH}}-i-1}. 
\]
% $\{x_{k|i+1}\}_{k=0}^{N_{\mathrm{SH}}-i-1}=\{x^{s*}_{1|i}, \dots, x^{s*}_{N_{\mathrm{SH}}-i|i}\}$
That is, we drop the first state vector of the contingency trajectory and reuse its suffix for both branches in the new problem. Thus, the candidate sequence satisfies the initial condition~\eqref{eq:init_states}, state constraint~\eqref{eq:problem_6}, terminal state constraint~\eqref{eq:problem_final}, and kinematics constraints~\eqref{eq:kinematics_r}--\eqref{eq:kinematics_s}. Also, the identity of two branches implies the consensus constraint~\eqref{eq:tie_5} holds trivially. 

From Lemma~\ref{lemma:merged_reach_containment} we have one-step containment relation $\hat{\mathcal{R}}_{k|i+1} \subseteq \hat{\mathcal{R}}_{k+1|i}$ for $k\in\mathcal{I}_0^{N_{\mathrm{SH}}-i-1}$. By Lemma~\ref{lem:safe-set-duality}, the associated safe sets are reversely contained: $\mathcal{F}(\hat{\mathcal{R}}_{k|i+1}) \supseteq \mathcal{F}(\hat{\mathcal{R}}_{k+1|i})$.
Since $x^{s*}_{k+1|i} \in \mathcal{F}(\hat{\mathcal{R}}_{k+1|i})$ is satisfied at time $t_i$, it follows that $x_{k|i+1} \in \mathcal{F}(\hat{\mathcal{R}}_{k|i+1})$. In other words, $\mathscr{B}(x_{k|i+1}, \hat{\mathcal{R}}^{h}_{k|i})\geq0$, satisfying FRS-based barrier constraint \eqref{eq:reachable_safety} again with $\alpha=1$.

% Since $\hat{\mathcal{R}}_{k|i+1} \subseteq \hat{\mathcal{R}}_{k+1|i}$ as a consequence of one-step containment relation by Lemma \ref{lemma:merged_reach_containment}, we then have:
% $\mathcal{F}(\hat{\mathcal{R}}_{k|i+1}) \supseteq   \mathcal{F}(\hat{\mathcal{R}}_{k+1|i})$ according to Corollary \ref{}.
% Since $\{x^{s*}_{1|i}, \dots, x^{s*}_{N_{\mathrm{SH}}-i|i}\}$ is part of the optimal trajectory at time $t_i$, we have:
% $x^{s*}_{k+1|i} \in \mathcal{F}(\hat{\mathcal{R}}_{k+1|i}) \subseteq \mathcal{F}(\hat{\mathcal{R}}_{k|i+1}), \forall k\in \mathcal{I}_0^{N_{\mathrm{SH}}-i-1}.$ 
% Thus, $x_{k|i+1} \in \mathcal{F}(\hat{\mathcal{R}}^{h}_{k|i+1})$ for all $k \in \mathcal{I}_0^{N_{\mathrm{SH}}-i-1}$ is satisfied at time $t_{i+1}$. In other words, $\mathscr{B}(x_{k|i+1}, \hat{\mathcal{R}}^{h}_{k|i})>0$, satisfying the FRS-based barrier constraint \eqref{eq:reachable_safety} again with $\alpha=1$.

Since we have $\hat{\mathcal{O}}^{h}_{k|i+1} \subseteq \hat{\mathcal{R}}^{h}_{k|i+1}$ under Assumption~\ref{assum:predict_and_reachable}, we have $x_{k|i+1}\in \mathcal{F}(\hat{\mathcal{R}}^{h}_{k|i+1}) \subseteq \mathcal{F}(\hat{\mathcal{O}}^{h}_{k|i+1}), \forall k\in \mathcal{I}_0^{N_{\mathrm{SH}}-i}$. It means the nominal safety constraint~\eqref{eq:nominal_safety} is satisfied with $\mathscr{B}(x_{k|i+1}, \hat{\mathcal{O}}^{h}_{k|i})\geq0$ with $\alpha=1$.

According to Corollary~\ref{coroll:update_time}, the actual HV control inputs up to the next update time remain contained in the learned control-intent sets. Thus, the actual HV occupancy $\mathcal{O}^h_{k|i+1}$ is contained in $\hat{\mathcal{R}}^h_{k|i+1}$, preserving the safety of open-loop trajectory, $x_{k|i+1}\in \mathcal{F}(\hat{\mathcal{R}}^{h}_{k|i+1}) \subseteq \mathcal{F}(\mathcal{O}^{h}_{k|i+1}), \forall k\in \mathcal{I}_0^{N_{\mathrm{SH}}-i-1}$. Since the closed-loop system executes the first state vector of the planned trajectory at each iteration, thus safety of the closed-loop trajectory is also satisfied, $x_i=x_{0|i} \in \mathcal{F}(\mathcal{O}^h_{0|i})=\mathcal{F}(\mathcal{O}^h_i)$.

Let $P_x = [p^{s*}_{x,1|i},\dots,p^{s*}_{x,N_{\mathrm{SH}}-i|i}]^T$ collect the longitudinal components of the shifted contingency sequence, and analogously for $P_y$ and $P_\theta$. Let $\mathcal{B}=[\mathcal{B}_1\,\dots\,\mathcal{B}_{N_{\mathrm{SH}}-i}] \in \mathbb{R}^{(n+1)\times (N_{\mathrm{SH}}-i)}$ denote the Bernstein basis evaluation matrix; $n \geq N_{\mathrm{SH}}-1$ implies $\mathcal{B}$ has full column rank, so the minimum-norm pseudoinverse yields control points $c_x^{s*} = \mathcal{B}^{\dagger T}P_x$ that reproduce the candidate trajectory. Repeating each component gives valid Bézier parameterizations satisfying \eqref{eq:bezier_r}--\eqref{eq:bezier_s}.

% To ensure the existence of the control points corresponding to the candidate trajectory, we note that the Bézier parameterization allows for solving these states using the Bernstein basis matrix $\mathcal{B}=\bigl[\mathcal{B}_1\,\dots\,\mathcal{B}_N\bigr]
%      \in\mathbb{R}^{(n+1)\times N}$, which has full column rank and thus is always invertible $BB^{\dagger}=I_{n+1}$ given that the degree $n> N$. For the longitudinal displacements $P_x=[p^{s*}_{x,1|i}\,  \dots\,  p^{s*}_{x,N_{\mathrm{SH}}-i|i}]^T$ of candidate trajectory, solving
% $\mathcal{B}^Tc^{s*}_x=P_x$ by the minimum-norm pseudoinverse gives $c^{s*}_x = B^{\dagger T}P_x $, which is a valid set of control points. 
% Applying this to each state component yields
% control points $c^{s*}_x$, $c^{s*}_y$, and $c^{s*}_\theta$ that reproduce the candidate trajectory.

This establishes recursive feasibility and both safety properties at $t_{i+1}$. Induction on $i$ completes the proof.
\end{proof}

% \begin{remark}\label{rem:tracking_tube_extension}\vspace{-2mm}
% If the EV tracking error is bounded by a tube $\mathcal{E}_\epsilon$, all safety conditions above remain valid after inflating each EV occupancy set by $\mathcal{E}_\epsilon$. We omit the resulting notation for brevity.%%AUTHOR: add sec reference.
% \end{remark}

% \begin{remark}[Why $N_s = N$ Does Not Trivialize Contingency]
% In the finite-mission shrinking case we enforce $N_s = N$ so that the nominal and contingency branches share the full trajectory. This reflects the fact that contingency action (e.g., mandatory lane change) cannot be deferred. Safety is nevertheless certified through the reachable-set barriers; if replanning fails, executing the already-computed sequence remains safe by Theorem~\ref{thm:shrink_rec_feas}. In normal driving we revert to $N_s < N$; see the receding-horizon case below.
% \end{remark}

\paragraph*{\textbf{Receding-Horizon Configuration}}
For routine navigation in dynamic environments, we employ a fixed horizon:
    \begin{equation}
    \vspace{-1mm}
    \label{eq:receding_horizon_n}
        N = N_{\mathrm{RH}}
        \vspace{-1mm}
    \end{equation}
which shifts forward one step at each planning instant. This allows continuous incorporation of updated HV intent and FRS information and is the standard operating mode of our framework.

% Before establishing the receding-horizon recursive feasibility result, we formalize a terminal FRS condition that bounds the influence of behaviors beyond the prediction horizon.
% \begin{assumption}\label{assum:equal_frs}
% At the planning instant $t_i$, the HV FRS position occupancy beyond the prediction horizon remains conservatively constant:
% \[
%     \hat{\mathcal{R}}^h_{k|i} = \hat{\mathcal{R}}^h_{N|i}, \quad \forall k \ge N.
% \]
% Moreover, $\mathcal{F}(\hat{\mathcal{R}}^h_{N|i})$ is forward invariant for the EV under a designated fallback policy (e.g., controlled braking).% against all HV behaviors admissible by $\hat{\mathcal{R}}^h_{N|i}$.
% \end{assumption}

\begin{assumption}\label{assum:equal_frs}
The terminal state set $\mathcal{X}_f$ is forward invariant for the EV under a designated fallback policy (e.g., controlled braking), ensuring $\mathcal{X}_f\subseteq \mathcal{F}(\hat{\mathcal{R}}^h_{k|i}),\, \forall k \ge N$, with the HV FRS positional occupancy beyond the prediction horizon.
\end{assumption}

\begin{remark}
Assumption~\ref{assum:equal_frs} reflects the practical impossibility of predicting HV intent arbitrarily far ahead in mixed traffic. Consequently, we assume that HVs beyond the horizon will not engage in excessively irrational or deliberately adversarial maneuvers that would preclude the forward invariance of the terminal state set. The forward invariance can be enforced by properly employing a well-designed fallback policy, such as controlled braking, with the barrier constraints as established in Lemma~\ref{thm:forward_invariance}. Similar set invariance conditions are widely adopted in model predictive control literature to guarantee recursive feasibility~\cite{kerrigan2000invariant, fang2022recursive, chen2023invariant}. %Furthermore, the influence of future behaviors diminishes rapidly, which be well addressed through receding horizon mechanism.  
\end{remark}
\begin{theorem}[Receding-Horizon Recursive Feasibility and Safety]\label{thm:rec_rec_feas}
% Under Assumptions~\ref{assum:predict_and_reachable}--\ref{assum:equal_frs}, 
Suppose: (i) the interval between successive HV FRS update events exceeds the planning horizon, i.e., $t_{m'} > t_{m+N}$; and (ii) Problem~\eqref{problem} with horizon~\eqref{eq:receding_horizon_n} is feasible at $t_m$ for Bézier curve order $n \geq N_{\mathrm{RH}}$ and $\alpha = 1$. Then for every $t_i \in [t_m, t_{m'-N}]$, the following hold:
\begin{enumerate}
    \item \textbf{Recursive Feasibility:} Problem~\eqref{problem} remains feasible.
    \item \textbf{Open-Loop Safety:} The contingency trajectory satisfies $x^s_{k|i} \in \mathcal{F}(\mathcal{O}^h_{k|i})$ for all $k \in \mathcal{I}_0^N$.
    \item \textbf{Closed-Loop Safety:} The executed trajectory satisfies $x_i \in \mathcal{F}(\mathcal{O}^h_i)$ at each time step.
\end{enumerate}
\end{theorem}
\begin{proof}

Let $\{x^{r*}_{0|i},\dots,x^{r*}_{N|i}\}$ and $\{x^{s*}_{0|i},\dots,x^{s*}_{N|i}\}$ denote the optimal nominal and contingency state sequences obtained from~\eqref{problem} at time $t_i$.
The argument generalizes the proof of Theorem~\ref{thm:shrink_rec_feas} with the following additions.

Under Assumption~\ref{assum:equal_frs}, there exists a state vector $x^{s*}_{N+1|i} \in \mathcal{X}_f\subseteq\mathcal{F}(\hat{\mathcal{R}}^h_{N+1|i})$ satisfying the kinematic constraint~\eqref{eq:kinematics_s}. 
At time instant $t_{i+1}$, we consider 
\[
    x_{k|i+1} = x^{s*}_{k+1|i}, \; k=\mathcal{I}_0^{N}.
\]
as a trajectory candidate for both nominal and contingency branches. Since the predicted HV FRS position occupancy beyond  horizon $N$ remains constant at time instant $t_i$, $x^{s*}_{N+1|i}\in \mathcal{F}(\hat{\mathcal{R}}^h_{N+1|i}) \subseteq \mathcal{F}(\hat{\mathcal{R}}^h_{N|i+1})$, so the FRS-based barrier constraint~\eqref{eq:reachable_safety} holds one step later, i.e., $x_{k|i+1} \in \mathcal{F}(\hat{\mathcal{R}}^{h}_{k|i+1})$ for all $k \in \mathcal{I}_0^{N}$. Lemma~\ref{lemma:merged_reach_containment} and Lemma~\ref{lem:safe-set-duality} propagate the remaining steps exactly as in the shrinking-horizon case. Also, the control points of B\'ezier curves that reproduce the candidate trajectories can be computed following the same way. \qedhere
\end{proof}

\begin{remark}
\label{remark:bezier_order}
While the assumption $n \geq N_{\mathrm{RH}}$ (or $n \geq N_{\mathrm{SH}}-1$ in Theorem~\ref{thm:shrink_rec_feas}) streamlines the proof algebraically, high-order Bézier curves are generally not desirable in practice due to their tendency to exhibit oscillatory behavior, which can induce control chattering and numerical instabilities \cite{Farin2002}. Instead, trajectories are typically constructed as piecewise low-order Bézier segments with continuity constraints at junctions. Crucially, the recursive feasibility guarantee can be easily extended to remain valid with proper trajectory parameterization. This depends primarily on the horizon-shifting mechanism and the properties of the invariant safety barrier constraints in the contingency mechanism.
\end{remark}

Notably, the guarantees in Theorems~\ref{thm:shrink_rec_feas}--\ref{thm:rec_rec_feas} are conditional on the learned control-intent sets, as discussed in Remark~\ref{rem:adaptive_safety}.

\begin{remark}
\label{rem:adaptive_safety}
Lemma~\ref{lem:robust_forward_invariance} and Theorems~\ref{thm:shrink_rec_feas}--\ref{thm:rec_rec_feas} establish the theoretical internal consistency, encompassing robust safety and recursive feasibility conditioned on the learned control-intent sets.
This avoids the excessive conservatism of worst-case physical bounds or the inadequacy of static assumptions that may fail to cover dynamic maneuvers.
This theoretical guarantee is bridged to reality by the event-triggered learning mechanism, which dynamically monitors and adapts to unmodeled maneuvers.  %It reflects a pragmatic paradigm of ensuring safety with respect to the observed behaviors while maintaining the capability to adapt to deviations.
\end{remark}

% \textcolor{blue}{The proposed framework achieves non-conservative yet safe planning by integrating online learning of control-intent set with a contingency planning architecture. While the online learning mechanism reduces conservatism with tighter uncertainty bounds, the architecture structurally allows the nominal plan to pursue agile and efficient maneuvers based on the most likely prediction, without being uniformly restricted by the FRS barrier constraints for the entire future horizon.}

% \textcolor{blue}{Crucially, the proposed framework achieves non-conservative yet safe planning by integrating this online learning-based safety paradigm with the contingency planning architecture. 
% While the event-triggered learning reduces conservatism by providing tighter uncertainty bounds, the architecture structurally allows the nominal plan to pursue agile and efficient maneuvers, without being uniformly restricted by the FRS barrier constraints for the entire future horizon.}

Crucially, this adaptive foundation synergizes with the contingency architecture to achieve nonconservative planning. 
While the event-triggered learning reduces conservatism by providing tighter uncertainty bounds, the architecture structurally allows the nominal trajectory to pursue agile and efficient maneuvers, without being uniformly restricted by the FRS barrier constraints for the entire future horizon.

% \begin{remark}
% The barrier coefficient $\alpha \in (0,1]$ controls the adjustment rate of barrier function $\mathscr{B}$ during planning. Smaller $\alpha$ values impose stricter temporal constraints, promoting proactive collision avoidance. With $\alpha = 1$, the safety constraint reduces to the minimally restrictive form $\mathscr{B}\geq0$, which is sufficient to guarantee recursive feasibility for arbitrary obstacle realizations according to Theorems~\ref{thm:shrink_rec_feas}--\ref{thm:rec_rec_feas}. In practice, \(\alpha\) effectively balances strict trajectory feasibility with proactive collision avoidance performance. Our experiments in Section~\ref{sec:Results} demonstrate that the proposed method still retains superior feasibility compared to alternative approaches even with $\alpha < 1$. 
% \end{remark}

 \vspace{-4mm}
\subsection{Contingency Planning with Consensus ADMM}
\label{subsec:admm}
% Building upon the optimization problem \eqref{problem}, we present a compact and equivalent reformulation, identify its biconvex structure, and solve it efficiently by ADMM. 

\subsubsection{Reformulation of Trajectory Optimization}
Building upon the optimization problem \eqref{problem}, we present a compact and equivalent reformulation as follows: 
\begin{subequations}    
    \label{problem3}
      \begin{align}  
      \displaystyle\operatorname*{min}_{\substack{\{c_\theta, c_x,c_y\}
     \\
    \{{\omega}, {d}\}}}~~ 
    \ &f(c_\theta)  + f_x(c_x) + f_y(c_y) \label{eq:problem3_pto1}\\
    \operatorname*{s.t.}\quad\quad
    & F_0  [c_{x}\ c_{y}]-[E_{x,0}\ E_{y,0}]=0,\label{eq:problem3_init}\\
    % & F_{\theta}  c_{\theta}-E_{\theta,0}=0, \label{eq:problem3_init_theta}\\
    & F_{\theta,0}  c_{\theta}=E_{\theta,0}, \, F_{\theta,f}  c_{\theta}=E_{\theta,f},\label{eq:problem3_init_fin_theta}\\ 
    & {\omega}  \in \mathcal{C}_{\omega}, {d}  \in \mathcal{C}_d, \label{eq:problem3_pto4}\\
     \mathcal{G}_{k,x} \triangleq\,& \dot{\mathcal{B}}^T c_{x}  - \mathbf{v} \cdot \cos{  \mathcal{B}^Tc_{\theta}}  =  0, \label{eq:problem3_pto5}\\     
    \mathcal{G}_{k,y} \triangleq\,&   \dot{\mathcal{B}}^T c_{y}  - \mathbf{v} \cdot \sin{  \mathcal{B}^Tc_{\theta}}  =  0, \label{eq:problem3_pto6}\\    
   \mathcal{G}_{bo,x} \triangleq\,& \mathcal{B}^T c^r_{x} - \mathcal{O}_{x}  - {l}^r_x \cdot  {d}^r \cdot  \cos{{\omega}^r} =  0, \label{eq:problem3_pto7}\\
    \mathcal{G}_{bo,y} \triangleq\,&\mathcal{B}^T c^r_{y} - \mathcal{O}_{y}  - {l}^r_y \cdot  {d}^r \cdot  \sin{{\omega}^r} =  0, \label{eq:problem3_pto8}\\ 
    \mathcal{G}_{br,x} \triangleq\,& \mathcal{B}^T c^s_{x} - \mathcal{R}_{x}  - {l}^s_x \cdot  {d}^s \cdot  \cos{{\omega}^s} =  0, \label{eq:problem3_pto7_s}\\
    \mathcal{G}_{br,y} \triangleq\,&\mathcal{B}^T c^s_{y} - \mathcal{R}_{y}  - {l}^s_y \cdot  {d}^s \cdot  \sin{{\omega}^s} =  0, \label{eq:problem3_pto8_s}\\ 
    \mathcal{G}_{cs,x} \triangleq\,& {A}^T_{c,x} c_x - {Y}_{x}=0,  ,\label{eq:problem2_pto8}  \\ 
     \mathcal{G}_{cs,y} \triangleq\,&   {A}^T_{c,y}  c_y - {Y}_{y}   =0 ,\label{eq:problem2_pto9}  \\ 
      \mathcal{G}_{cs,\theta} \triangleq\,&  {A}^T_{c,\theta}  c_{\theta}  - {Y}_{\theta}=0 ,\label{eq:problem2_pto10}  \\
    \mathcal{G}_{c,x} \triangleq\,&{G} c_x  -{h}_x \leq  0,	\label{eq:problem3_pto9} \\
    \mathcal{G}_{c,y} \triangleq\,&{G} c_y - {h}_y  \leq  0.	\label{eq:problem3_pto10}   
\end{align}
\end{subequations}
%  todo: not a scalar \mathcal{G}_{kino} \triangleq\,
Here, $\triangleq$ denotes the definition operator for convenience. The optimization variables are denoted as $c_x=[c^r_x \ c_x^s]\in\mathbb{R}^{(n+1)\times 2},c_\theta=[c^r_\theta \ c_\theta^s]\in\mathbb{R}^{(n+1)\times 2}$, and $c_y=[c^r_y \ c_y^s]\in\mathbb{R}^{(n+1)\times 2}$. 
The variables \({\omega} = [{\omega}^r\ {\omega}^s] \in \mathbb{R}^{(N \times M)\times 2}\) correspond to the matrix stacking relative angle \(\omega^{h}_k\), for \(M\) HVs within the planning horizon \(N\), while \({d} = [{d}^r\ {d}^s] \in \mathbb{R}^{(N \times M) \times 2}\) follows an analogous structure.
The cost function \eqref{eq:problem3_pto1} comprises three components:
\begin{subequations}
\small
\begin{align}
&f(c_\theta) = \frac{1}{2}c_\theta^T Q_\theta c_\theta \label{eq:fx_simple}, \\
&f(c_x) = \frac{1}{2}c_x^T Q_x c_x + \frac{1}{2}c_x^T \dot{\mathcal{B}} Q_{x,v} \dot{\mathcal{B}}^T c_x - v_{x,d}^T Q_{x,d} \dot{\mathcal{B}}^T c_x, \label{eq:fx} \\
&f(c_y) = \frac{1}{2}c_y^T Q_{y} c_y + \frac{1}{2}c_y^T \dot{\mathcal{B}} Q_{y,v} \dot{\mathcal{B}}^T c_y - v_{y,d}^T Q_{y,d} \dot{\mathcal{B}}^T c_y, \label{eq:fy}
\end{align}
\end{subequations}
where $Q_\theta$,  $Q_x$, $Q_y$, $Q_{x,v}$, $Q_{x,d}$,  $Q_{y,v}$, and $Q_{y,d}$ denote symmetric positive semi-definite weighting matrices.

The constraint matrices for the initial and terminal state constraints \eqref{eq:problem3_init}-\eqref{eq:problem3_init_fin_theta} are given by $F_0=[\mathcal{B}_0\, \dot{\mathcal{B}}_0\, \ddot{\mathcal{B}}_0]^T\in\mathbb{R}^{3\times(n+1)}$, $F_{\theta,0}=[\mathcal{B}_0\, \dot{\mathcal{B}}_0]^T\in\mathbb{R}^{2\times(n+1)}$, and $F_{\theta,f}=[\mathcal{B}_N\, \dot{\mathcal{B}}_N]^T\in\mathbb{R}^{2\times(n+1)}$, respectively. The vectors \( {E}_{x,0}, {E}_{y,0} \in \mathbb{R}^{3\times2} \) encode the initial position, velocity and acceleration constraints along longitudinal and lateral components for both nominal and contingency trajectories, while ${E}_{\theta,0}, \, {E}_{\theta,f} \in \mathbb{R}^{2\times2}$ encode the initial and terminal state vectors for orientation angles and angular velocities.

The kinematic constraints \eqref{eq:problem3_pto5} and \eqref{eq:problem3_pto6} derived from  \eqref{eq:kinematics_r} and \eqref{eq:kinematics_s}, respectively, are enforced through the composite speed matrix \(\mathbf{v} = [\mathbf{v}^r\  \mathbf{v}^s] \in \mathbb{R}^{N\times 2}\), where $\mathbf{v}^r$ and $\mathbf{v}^s$ represent the speed profiles for the nominal and contingency trajectories over the $N$-step horizon.

The collision avoidance constraints \eqref{eq:problem3_pto7}-\eqref{eq:problem3_pto8} and \eqref{eq:problem3_pto7_s}-\eqref{eq:problem3_pto8_s} correspond to the safety conditions \eqref{eq:nominal_safety} and \eqref{eq:reachable_safety}, respectively. The predicted longitudinal and lateral positions of HVs for the nominal trajectory are denoted by \(\mathcal{O}_{x} ,\mathcal{O}_{y} \in \mathbb{R}^{N \times M}\) while \(\mathcal{R}_{x} ,\mathcal{R}_{y} \in \mathbb{R}^{N \times M}\) represent the corresponding FRS occupancy centers for the contingency trajectory. The safety ellipsoid dimension \({l}^r_x\in \mathbb{R}^{N\times M}\) is constructed by stacking \(l^{h}_x\) along the nominal trajectory, while ${l}^r_y, {l}^s_x, {l}^s_y$ are constructed analogously.

For the state constraints \eqref{eq:problem3_pto9} and \eqref{eq:problem3_pto10} implementing \eqref{eq:problem_6}, the constraint matrix ${G} = [\mathcal{B}^T\ -\mathcal{B}^T\ \ddot{\mathcal{B}}^T\ -\ddot{\mathcal{B}}^T\ \dddot{\mathcal{B}}^T\ -\dddot{\mathcal{B}}^T]^T\in \mathbb{R}^{6N\times (n+1)}$ is constructed from $\mathcal{B}=[\mathcal{B}_1\, \cdots \,\mathcal{B}_N]\in\mathbb{R}^{(n+1)\times N}$. The matrices \({h}_x,{h}_y \in \mathbb{R}^{6N\times 2}\) contain the corresponding state bounds for longitudinal and lateral motions, respectively.

% for smoothness and consistency
The constraints \eqref{eq:problem2_pto8}-\eqref{eq:problem2_pto10} enforce trajectory consistency for position, velocity, and acceleration at the initial segments between two trajectories. The corresponding matrices are ${A}_{c,x} ={A}_{c,y} = [\mathcal{B}_{[1:N_s]}^T\ \dot{\mathcal{B}}_{[1:N_s]}^T\ \ddot{\mathcal{B}}_{[1:N_s]}^T]^T\in \mathbb{R}^{((n+1) \times 3)  \times N_s}$ and ${A}_{c,\theta} = \mathcal{B}_{[1:N_s]} \in \mathbb{R}^{(n+1) \times N_s}$. For the consensus variables, each column of \({Y}_{x} \in \mathbb{R}^{3N_s \times N_c}\), \({Y}_{y} \in \mathbb{R}^{3N_s \times N_c}\), and \({Y}_{\theta} \in \mathbb{R}^{N_s \times N_c}\) maintains consistency across candidate trajectories during ADMM iterations.

\subsubsection{ADMM Iterations}
The barrier constraints \eqref{eq:problem3_pto7}--\eqref{eq:problem3_pto8_s} exhibit biconvexity in $[c_{x}\, c_{y}\, {d}]^T$ and $[\cos\omega\, \sin\omega]^T$. This biconvex structure allows decomposition of the joint constraints \eqref{eq:problem3_pto7}-\eqref{eq:problem3_pto8_s} into two alternating optimizations: (i) Position optimization of $[c_x\, c_y]^T$ and ${d}$ with fixed ${\omega}$; (ii) Angular optimization of ${\omega}$ with fixed $[c_x\, c_y]^T$ and ${d}$. Similarly, the kinematic constraints \eqref{eq:problem3_pto5}-\eqref{eq:problem3_pto6} admit an analogous decomposition, alternating between position variables $[c_x\, c_y]^T$ (fixed $c_\theta$) and heading angle $c_\theta$ (fixed $[c_x\, c_y]^T$). Given that the quadratic objective \eqref{eq:problem3_pto1} preserves convexity, this constraint structure enables efficient ADMM-based decomposition of \eqref{problem3} into quadratic programming (QP) subproblems.

We derive an augmented Lagrangian of \eqref{problem3} for ADMM iterations, as follows:  
\begin{align} \vspace{-0mm} 
&\mathcal{L}({ c_x, c_y, c_{\theta} }, {{Z}_x, {Z}_y},   {{\omega}, {d}},  {Y}_x,{Y}_y,{Y}_{\theta} \nonumber\\
& \quad   \quad {\lambda}_{x}, {\lambda}_{y}, {{\lambda}_{\theta}, {\lambda}_{\text{obs},x}, {\lambda}_{\text{obs},y}, {\lambda}_{c,x}, {\lambda}_{c,y}, {\lambda}_{c,\theta}} 
)  \nonumber\\
    &= f_x(c_x) + f_y(c_y) +f(c_\theta) +  \mathcal{I}_{+}({Z}_x) + \mathcal{I}_{+}({Z}_y) \nonumber\\ 
    &\quad+ {\lambda}^T_{x}c_x  + {\lambda}^T_{y}c_y + \rho_x \left\| \mathcal{G}_{c,x}  + {Z}_x \right\|_{2}^{2} + \rho_y \left\| \mathcal{G}_{c,y} + {Z}_y \right\|_{2}^{2}\nonumber \\
    &\quad+ \rho_{\theta}  \left\| \mathcal{G}_{k,x}
    + \frac{{\lambda}_{\theta} }{\rho_{\theta}}
    \right\|_{2}^{2} 
    + \rho_{\theta}  \left\| \mathcal{G}_{k,y}
    +\frac{{\lambda}_{\theta} }{\rho_{\theta}}
    \right\|_{2}^{2} \nonumber \\
    &\quad+ \rho^r_{\text{obs}}  \left\| \mathcal{G}_{bo,x} 
    + \frac{ {\lambda}^r_{\text{obs},x}}{\rho^r_{\text{obs}}} \right\|_{2}^{2} 
       + \rho^r_{\text{obs}}  \left\| \mathcal{G}_{bo,y} 
    + \frac{{\lambda}^r_{\text{obs},y}}{\rho^r_{\text{obs}}} \right\|_{2}^{2} 
       \nonumber\\
       &\quad+ \rho^s_{\text{obs}}  \left\| \mathcal{G}_{br,x} 
    + \frac{ {\lambda}^s_{\text{obs},x}}{\rho^s_{\text{obs}}} \right\|_{2}^{2} 
       + \rho^s_{\text{obs}}  \left\| \mathcal{G}_{br,y} 
    + \frac{{\lambda}^s_{\text{obs},y}}{\rho^s_{\text{obs}}} \right\|_{2}^{2} 
       \nonumber\\
    &\quad+ \rho_{c,x }  \left\| \mathcal{G}_{cs,x} 
     + \frac{ {\lambda}_{c,x }}{\rho_{c,x}}  \right\|_{2}^{2} 
   + \rho_{c, y }  \left\| \mathcal{G}_{cs,y} 
     + \frac{ {\lambda}_{c,y }}{\rho_{c,y}}  \right\|_{2}^{2} \nonumber\\
       &\quad+ \rho_{c,\theta }  \left\| \mathcal{G}_{cs,\theta}      
      + \frac{ {\lambda}_{c,\theta }}{\rho_{c,\theta}}  \right\|_{2}^{2}.
    \label{lang_dual_problem}
\vspace{-0mm} \end{align}

\begin{figure*}[tp]
    \centering \hspace{-4mm}
    % First row of subfigures
    \subfigure[]{
        \includegraphics[scale=0.59]{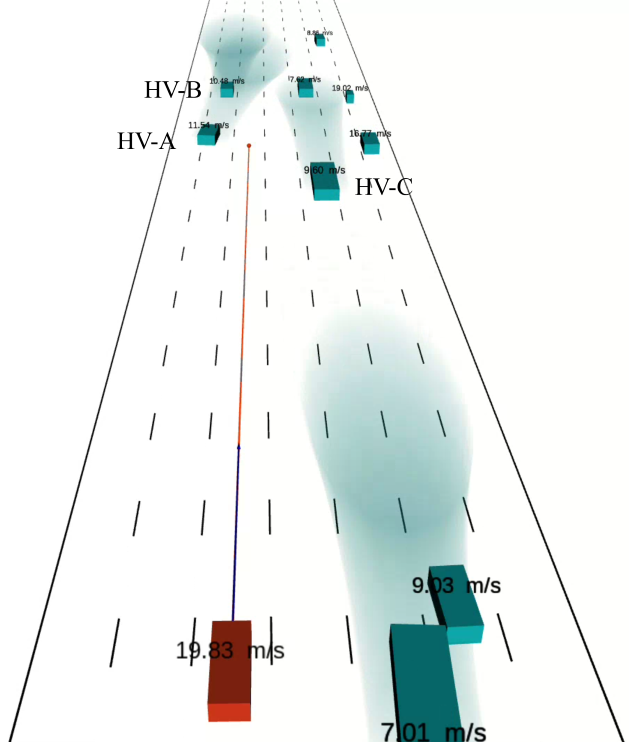}
    }
    \subfigure[]{
        \includegraphics[scale=0.59]{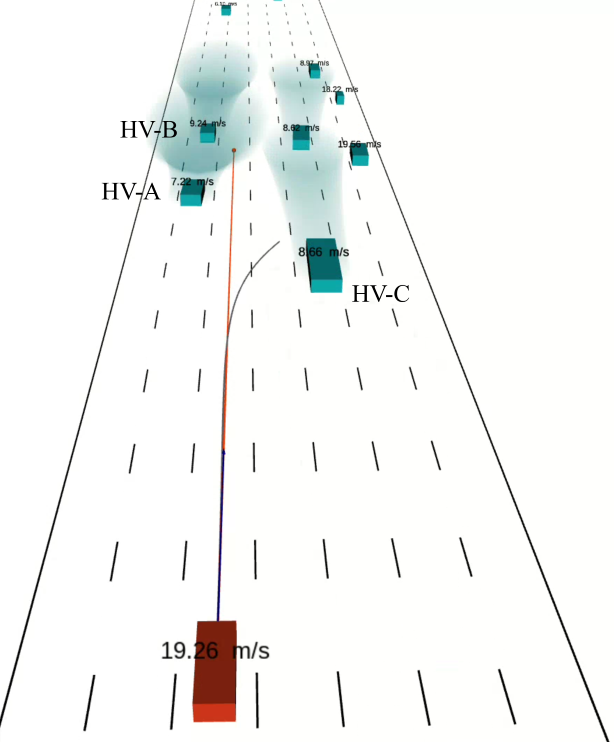}
    }
    \subfigure[]{
        \includegraphics[scale=0.59]{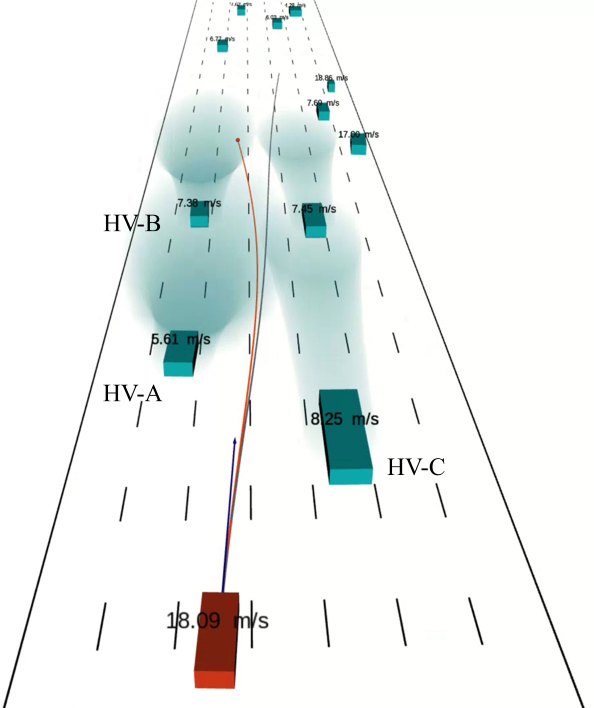}
    }
    \subfigure[]{
        \includegraphics[scale=0.59]{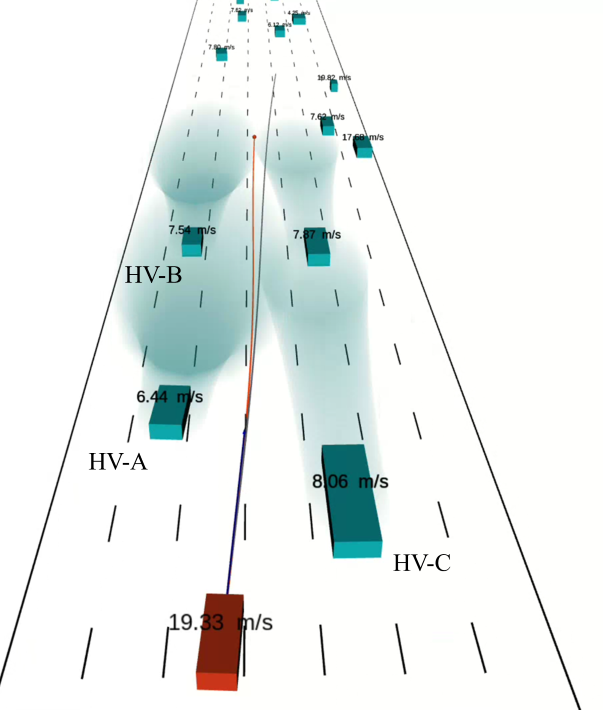}
    }
    \subfigure[]{
        \includegraphics[scale=0.59]{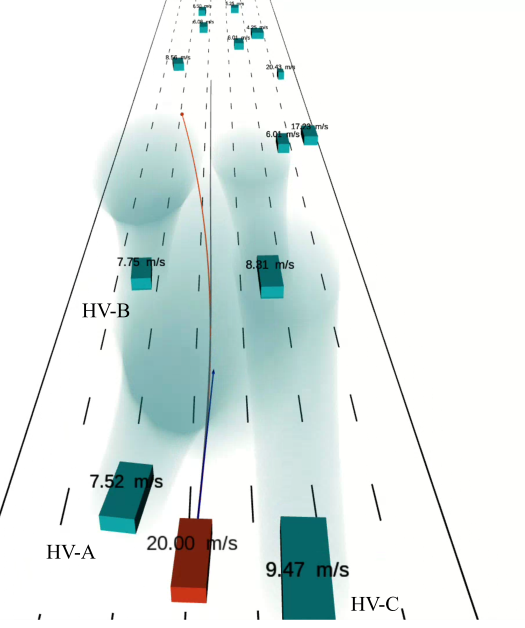}
    }
    \subfigure[]{
        \includegraphics[scale=0.59]{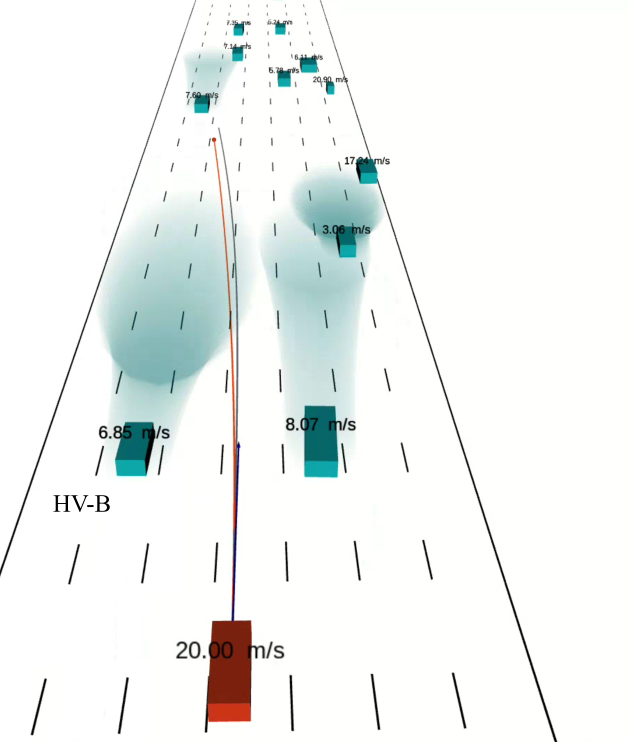}
    }
    \vspace{-4mm}
    \caption{Snapshots of the EV navigating dense traffic in the NGSIM dataset. The EV (red) jointly optimizes a nominal trajectory (red) and a contingency trajectory (black), with blue arrows showing velocities.     
    (a)-(b) The EV dynamically captures the uncertain intentions of HVs with online updating FRSs (shaded regions) as it approaches them; (c)-(d) The EV proactively responds to the potential cut-in of HV-A with anticipatory deceleration and controlled rightward avoidance maneuvers; (e)-(f) The EV safely navigates past the cut-in HV-A.}
    \label{fig:snapshots_cruise_static}
    \vspace{-4mm}
\end{figure*}

In this formulation, the dual variables \( {\lambda}_{x}, {\lambda}_{y} \in \rr^{(n+1)\times 2} \) correspond to inequality constraints \eqref{eq:problem3_pto9}-\eqref{eq:problem3_pto10} for iteration stability, following~\cite{taylor2016training}; $\rho_{x}$, $\rho_{y}$ $\rho_{\theta}$, $\rho^r_{\text{obs}}$, and $\rho^s_{\text{obs}}$ denote the associated $l_2$ penalty parameters.  
The consensus dual variables ${\lambda}_{c,x} = [{\lambda}^r_{c,x}\, {\lambda}^s_{c,x}]\in \rr^{3N_s
\times 2}$, ${\lambda}_{c,y} = [{\lambda}^r_{c,y}\, {\lambda}^s_{c,y}]\in \rr^{3N_s
\times 2}$ and ${\lambda}_{c,\theta} = [{\lambda}^r_{c,\theta}\, {\lambda}^s_{c,\theta}] \in \rr^{N_s\times 2}$ enforce the constraints~\eqref{eq:problem2_pto8}-\eqref{eq:problem2_pto10};  $\rho_{c,x}$, $ \rho_{c,y}$, $\rho_{c,\theta}$ are the corresponding $l_2$ penalty parameters, driving local variables toward their consensus global values ${Y}_x$, ${Y}_y$, and $ {Y}_{\theta}$.  
The remaining dual variables
$ {\lambda}_{\theta}\in \rr^{N\times 2}$, ${\lambda}_{\text{obs},x} = [{\lambda}^r_{\text{obs},x}\, {\lambda}^s_{\text{obs},x}] \in \rr^{(N\times M)\times 2}$, and ${\lambda}_{\text{obs},y} = [{\lambda}^r_{\text{obs},y}\,  {\lambda}^s_{\text{obs},y}] \in \rr^{(N\times M)\times 2}$ are assigned to the constraints \eqref{eq:problem3_pto7}-\eqref{eq:problem3_pto8_s}.  We introduce slack variables $Z_x, Z_y$ to handle potential constraint violations in \eqref{eq:problem3_pto9} and \eqref{eq:problem3_pto10}. The indicator function $\mathcal{I}_{+}(\cdot)$ enforces exact constraint satisfaction: $\mathcal{I}_{+}(Z) = 0 \text{ if } Z \geq 0, \text{and}\ \infty \text{ otherwise}$.

Following the ADMM framework~\cite{zheng2024barrier}, the augmented Lagrangian function~\eqref{lang_dual_problem} naturally partitions the primal variables into five distinct groups for alternate optimization: 
heading angle $c_\theta$, longitudinal variables $\{{c}_x, {Z}_x\}$, lateral variables $\{{c}_y, {Z}_y\}$, angular variables ${\omega}$, and distance variables ${d}$.

\textbf{\textit{Primal Update.}}\quad 
The primal variables are updated sequentially by solving the following QP subproblems:
\begin{alignat}{2} 
    c^{\iota+1}_{\theta} &= \displaystyle\operatorname*{\text{argmin}}_{c_{\theta}}~~
    \mathcal{L} \Big(c_\theta, \{\cdot\}^\iota \Big) 
    \ \text{s.t.} \
    [F_{\theta,0}\,F_{\theta,f}] c_{\theta}  = [E_{\theta,0}\,E_{\theta,f}],  \nonumber
\end{alignat} 
\vspace{-3mm}
\begin{alignat}{2} 
    c^{\iota+1}_{x}  &= \displaystyle\operatorname*{ \text{argmin}}_{c_{x}}~~
    \mathcal{L}   \Big(c_x, c^{\iota+1}_{\theta}, \{\cdot\}^\iota \Big) 
      \ \text{s.t.}\
   F_0 c_{x}  = E_{x,0},    \nonumber 
\end{alignat} 
\vspace{-3mm}
\begin{alignat}{2} 
    c^{\iota+1}_{y}  &= \displaystyle\operatorname*{ \text{argmin}}_{c_{y}}~~
    \mathcal{L}   \Big( c_y, c^{\iota+1}_{x}, c^{\iota+1}_{\theta}, \{\cdot\}^{\iota} \Big)
       \ \text{s.t.}\
   F_0 c_{y}  = E_{y,0},    \nonumber 
\end{alignat} 
where $\{\cdot\}^\iota$ compactly represents all other variables of $\mathcal{L}$ fixed at iteration $\iota$.

% where $\mathbf{P}_{y,0} \in \mathbb{R}^{ N_c}$ and  $\dot{\mathbf{P}}_{y,0}\in\mathbb{R}^{ N_c}$ denote the initial lateral position and velocity for $N_c$ candidate trajectories, respectively.  
%  Note that $\mathbf{P}_{x,N} \in \mathbb{R}^{ N_c}$ and $\mathbf{P}_{y,N} \in \mathbb{R}^{ N_c}$ are target longitudinal and lateral position vectors for $N_c$ candidate trajectories, which are obtained using an adaptive sampling strategy detailed in Algorithm 1 of~\cite{zheng2024barrier}.  See Appendix for the analytical form of $ c^{\iota+1}_{\theta}$, $ c^{\iota+1}_{x}$, and $ c^{\iota+1}_{y}$. 

 To handle the slack variables $Z_x, Z_y$ for the inequality constraints \eqref{eq:problem3_pto9}-\eqref{eq:problem3_pto10}, the iteration follows the form
 \begin{equation}
       % {Z}^{\iota+1}_{\mathfrak{q}} = &\max  \Big( {0},  {F}_{\mathfrak{q}} - {G}c^{\iota+1}_{\mathfrak{q}}  \Big),{\mathfrak{q}}\in\{x,y\}   
Z^{\iota+1}_\bullet = \max(0, F_\bullet - G c^{\iota+1}_\bullet), \quad \bullet \in \{x,y\}
       \label{eq:z_x_relaxedadmm_update} \nonumber
       % {Z}^{\iota+1}_y = &\max  \Big( {0},  {F}_y - {G}c^{\iota+1}_y  \Big).   \label{eq:z_x_relaxedadmm_update}  
\end{equation} 
 to strictly enforce the inequality constraints~\cite{ghadimi2015optimal}.

By leveraging the conditions in \eqref{eq:obtain_omega} and \eqref{eq:barrier_cons_polar}, we derive analytical solutions {for \(\omega\)} and \(d\):
\begin{subequations}
\small
\begin{align}
\omega^{r,\iota+1} &= \arctan\left(\textcolor{black}{\frac{l^r_x \cdot ({\mathcal{B}^T} c^{\iota+1}_y - \mathcal{O}_y)}{l^r_y \cdot ({\mathcal{B}^T} c^{\iota+1}_x - \mathcal{O}_x)}} \right), \nonumber \\
\omega^{s,\iota+1} &= \arctan\left(\textcolor{black}{\frac{l^s_x \cdot (\mathcal{B}^T c^{\iota+1}_y - \mathcal{R}_y)}{l^s_y \cdot (\mathcal{B}^T c^{\iota+1}_x - \mathcal{R}_x)}}\right),\nonumber \\
d^{\iota+1} &= \max\Big(1, 1 + (1-\alpha)(d^\iota -1)\Big). \nonumber
\end{align}
\end{subequations}

\textbf{\textit{Consensus Variables Update.}} \quad
As established in~\cite{boyd2011distributed}, the dual variables for consensus constraints maintain the zero-mean property:
\begin{equation}
\lambda^{r,\iota}_{c,\bullet} + \lambda^{s,\iota}_{c,\bullet} = 0, \quad \bullet \in \{x,y,\theta\}.\nonumber
\end{equation}
This structural property persists through ADMM iterations. The global variables update according to:
\begin{equation}
\small
\begin{aligned}
Y^{\iota+1}_\bullet[:,0] = Y^{\iota+1}_\bullet[:,1] = \frac{1}{2}A^T_{c,\bullet}\Big(c^{\iota+1}_\bullet[:,0] + c^{\iota+1}_\bullet[:,1]\Big), \\
\quad \bullet \in \{x,y,\theta\}.\nonumber
\end{aligned}
\label{eq:consensus_x_admm_update}
\end{equation}

\textbf{\textit{Dual Variables Update.}}\quad 
The dual variables are updated through:
\begin{subequations}
\small
\begin{align}
&\lambda^{\iota+1}_\theta = \lambda^\iota_\theta + \rho_\theta \Big(\mathcal{B}^Tc^{\iota+1}_\theta - \arctan\Big(\frac{\dot{\mathcal{B}}^T c^{\iota+1}_y}{\dot{\mathcal{B}}^T c^{\iota+1}_x}\Big)\Big), \notag\\
&\lambda^{\iota+1}_\bullet = \lambda^\iota_\bullet + \rho_\bullet \Big((1-\alpha_\bullet)(Z^{\iota+1}_\bullet - Z^\iota_\bullet)    \notag\\& \quad\quad \quad \quad \quad + \alpha_\bullet (G c^{\iota+1}_\bullet - F_\bullet + Z^{\iota+1}_\bullet)\Big),  \notag\\
&\lambda^{r,\iota+1}_{\text{obs},\bullet} = \lambda^{r,\iota}_{\text{obs},\bullet} + \rho^r_{\text{obs}} \Big(\mathbf{v}^r c^{r,\iota+1}_\bullet - \mathcal{O}_\bullet  - l^r_\bullet d^{r,\iota+1} \cos(\omega^{r,\iota+1})\Big),  \notag\\
&\lambda^{s,\iota+1}_{\text{obs},\bullet} = \lambda^{s,\iota}_{\text{obs},\bullet} + \rho^s_{\text{obs}} \Big(\mathbf{v}^s c^{s,\iota+1}_\bullet - \mathcal{R}_\bullet  - l^s_\bullet d^{s,\iota+1} \cos(\omega^{s,\iota+1})\Big),  \notag\\
&\lambda^{\iota+1}_{c,\theta} = \lambda^\iota_{c,\theta} + \rho_{\theta} (A^T_{c,\theta} c^{\iota+1}_\theta - Y^{\iota+1}_\theta),  \notag\\
&\lambda^{\iota+1}_{c,\bullet} = \lambda^\iota_{c,\bullet} + \rho_{c,\bullet} (A^T_{c,\bullet} c^{\iota+1}_\bullet - Y^{\iota+1}_\bullet),  \quad\quad \quad \quad \quad \bullet \in \{x,y\}. \notag
% &\makebox[\linewidth][r]{ \bullet \in \{x,y\}. } \notag %  
\end{align}
\label{eq:dual_updates}
\vspace{-2mm}
\end{subequations}

% The iteration terminates when primal residual satisfies $\|r^\iota\|_2 \le \epsilon_{\text{pri}}=0.5$ and dual residual $\|s^\iota\|_2 \le \epsilon_{\text{dual}}=0.01$, or the maximum iteration count $\iota_{\max}=200$ is reached.} 

The iteration terminates when the primal residual $\|r^\iota\|_2 \le \epsilon_{\text{pri}}=0.5$ and dual residual $\|s^\iota\|_2 \le \epsilon_{\text{dual}}=0.01$, or $\iota_{\max}=200$. The primal residual $r^\iota$ concatenates the violations of all equality constraints encoded in the augmented Lagrangian \eqref{lang_dual_problem}, including kinematics, barrier safety, consensus, and slack consistency. The primal tolerance $\epsilon_{\text{pri}}$ applies to the aggregate $L_2$-norm, ensuring a small element-wise deviation. The dual residual $s^\iota$ quantifies the stationarity of the Lagrangian, defined by the change in slack variables $Z_x, Z_y$ and consensus $Y_x, Y_y, Y_\theta$ across iterations.

\begin{remark}
Our decomposition strategy isolates non-convexities, ensuring that each ADMM subproblem remains a strictly convex QP with a positive-definite Hessian. This guarantees unique closed-form updates and prevents numerical singularities, thereby promoting convergence to a stationary point. Numerical stability is further enhanced by warm-start initialization. In rare cases of divergence, a hierarchical safety mechanism ensures system safety by utilizing feasible suboptimal solutions or the collision-free tail of the preceding trajectory, with an emergency stop mechanism reserved as the last-resort safeguard.
\end{remark}

%====================
\vspace{-7mm}
\section{Results}\label{sec:Results}
%====================
This section presents a comprehensive evaluation of the proposed approach through both high-fidelity simulations and real-world experiments. The simulation study constructs an urban intersection scenario and utilizes safety-critical highway scenarios from the Next Generation Simulation (NGSIM) dataset\footnote{\url{https://data.transportation.gov/Automobiles/Next-Generation-Simulation-NGSIM-Vehicle-Trajector/8ect-6jqj}}, while physical validation is conducted using a 1:10-scale Ackermann-steering platform operating with HVs exhibiting diverse driving behaviors.

\vspace{-4mm}
%----------------------------------------
\subsection{Simulation}\label{subsec:Sim_results}
%----------------------------------------
% \subsubsection{Simulation Setup}
The simulations are implemented in C++ running on an Ubuntu 24.04 LTS system with an Intel Ultra 9 285H CPU with 16 cores and 32 threads. The visualization uses RViz in ROS Noetic with a 100 Hz communication rate, while the planning loop operates with a time step $\delta T = 0.08\,\text{s}$ and a planning horizon $N = 50$ for both nominal and safe contingency trajectories. Statistically, the solver executes an average of 155 iterations (range: 24--200) per cycle with a mean computation time of 34 ms, satisfying real-time requirements. Biconvex structure exploitation with ADMM enables reliable real-time optimization.

The diagonal elements of the weighting matrices in \eqref{eq:problem3_pto1} are set to $50$ for ${Q}_x,{Q}_y,{Q}_\theta,{Q}_{xv},{Q}_{yv}$, and $100$ for ${Q}_{xd},{Q}_{yd}$. For the ADMM solver, penalty parameters are empirically tuned for robust convergence: $\rho_\theta=5$, $\rho^r_{\text{obs}}=\rho^s_{\text{obs}}=10$, and $\rho_{c,\bullet}=5$ ($\bullet \in \{x,y,\theta\}$). The barrier coefficient parameter $\alpha$ is set to $0.8$. The HV geometric ellipsoid axes $l_x^h$ and $l_y^h$ are set to $6\,\text{m}$ and $4\,\text{m}$ in the intersection scenario, and set according to their actual shapes in the highway dataset. The B\'ezier curve order $n$ is set to $10$. Acceleration bounds of $\pm5\, \mathrm{m/s^2}$ are enforced in both longitudinal and lateral directions. The terminal angular velocity of the orientation is set to zero. The considered HV number is set as $M=4$. The consensus step $N_s$ is set to $5$. The control dataset is initialized with $\pm0.2\, \mathrm{m/s^2}$ for longitudinal acceleration bounds and $\pm0.1\, \mathrm{m/s^2}$ for lateral acceleration bounds, representing a deliberate under-approximation.
To handle perception uncertainties, a linear Kalman Filter estimates the HV state from position and velocity observations with noise covariance $\text{diag}(0.1, 0.1, 0.5, 0.5)$. The estimated state covariance is dynamically integrated into the FRS initialization \eqref{eq:init_reachable_set} at each cycle as a $3\sigma$ confidence bound.

% We evaluate five approaches:
% \begin{itemize}
%     \item (a) The proposed safe contingency planner;
%     \item (b) Deterministic Barrier Planner: The proposed planner uses only trajectory prediction-based deterministic safety barriers in both horizons. Specifically, the FRS-based barrier constraint \eqref{eq:reachable_safety} is replaced by the nominal barrier constraint \eqref{eq:nominal_safety};
%     \item (c) Worst-Case Barrier Planner: The proposed planner employs worst-case FRS-based safety barriers without online updates, where the FRS prediction in \eqref{eq:reachable_safety} is replaced with a worst-case FRS derived from a prior control-intent set with $\pm3\, \mathrm{m/s^2}$ accelerations;
%     \item (d) Uncertain-Aware Planner~\cite{zhou2025robustb}: The trajectory planner proposed in \cite{zhou2025robustb} retains the same FRS update parameters for fair comparison. It directly integrates FRSs as collision-avoidance constraints in an NMPC, solving by ACADO toolkit~\cite{Houska2011a} in C++;
%     \item (e) ST-RHC~\cite{zheng2024}: An NMPC method employing a deterministic constant-velocity model for collision-avoidance constraints. It is implemented via the ACADO toolkit~\cite{Houska2011a}, with multiple shooting and sequential quadratic programming for computationally efficient optimization.
% \end{itemize}  

We evaluate five approaches:
\begin{itemize}
    \item[(a)] \textbf{Proposed Safe Contingency Planner}: Our complete framework integrating online-updated FRS-based barrier constraints with contingency planning;
    
    \item[(b)] \textbf{Deterministic Barrier Planner}: A variant of our planner replacing the FRS-based barrier constraint \eqref{eq:reachable_safety} with nominal barrier constraint \eqref{eq:nominal_safety} in both nominal and contingency branches;
    
    \item[(c)] \textbf{Worst-Case Barrier Planner}: A conservative version using worst-case FRS predictions (derived from a prior control dataset with $\pm3\, \mathrm{m/s^2}$ acceleration bounds) instead of online-updated FRS in \eqref{eq:reachable_safety};

    \item[(d)] \textbf{ST-RHC}~\cite{zheng2024}: A baseline NMPC method using constant-velocity prediction for collision-avoidance constraints, solved efficiently through ACADO toolkit~\cite{Houska2011a} with multiple shooting and sequential quadratic programming;
    
    \item[(e)] \textbf{Uncertain-Aware Planner}~\cite{zhou2025robustb}: A baseline uncertain-aware planner embeds FRSs directly in NMPC solved with ACADO in C++, maintaining identical FRS update parameters for fair comparison.
\end{itemize}

The comprehensive evaluation incorporates four key performance metrics:
\begin{itemize}
    \item \textbf{Safety Performance}: Collision rate $P_\text{c}$ and the minimum distance to the nearest HV $d_{\text{min}}$;
    \item \textbf{Comfort}: Peak longitudinal/lateral jerk $J_{\text{x,max}}$, $J_{\text{y,max}}$;
    \item \textbf{Driving Efficiency}: Mean speed $v_{\text{mean}}$ and travel distance $s_{\text{mean}}$;
    \item \textbf{Computation}: Mean computation time $t_{\text{mean}}$.
\end{itemize}

\begin{figure}[tb]
\begin{center}
\includegraphics[width=1.\columnwidth]{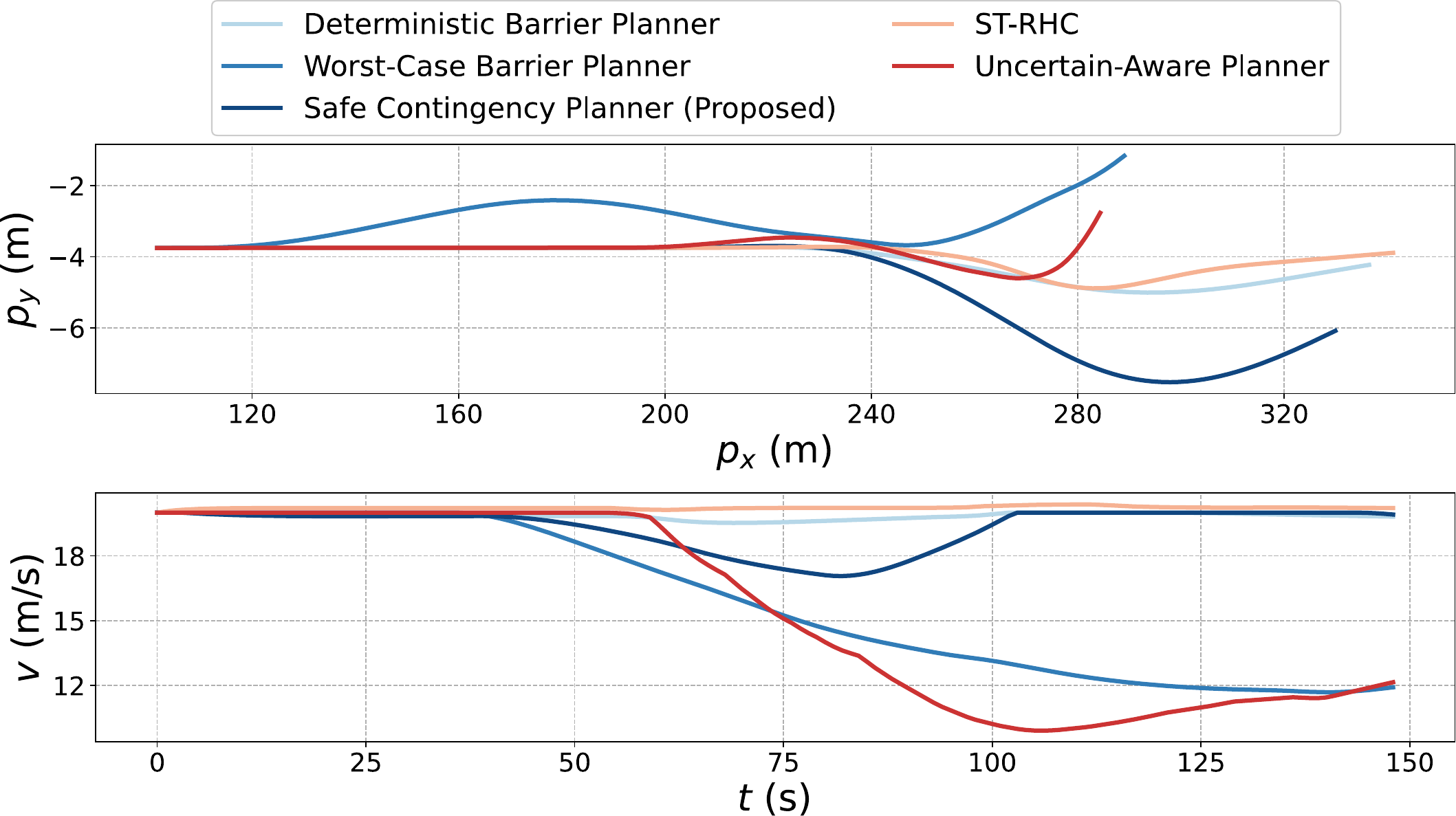}    \vspace{-5mm}
\caption{Trajectory and speed profiles of the EV under different planning methods in the highway scenario. Our method proactively adjusts speed and starts lateral evasion early to defend HV-A's uncertain maneuvers without excessive conservatism.
} \vspace{-4mm}
\label{fig:scene_1_trajectory}
\end{center}
\end{figure} 

\begin{figure}[tb]
\begin{center}
\includegraphics[width=1.\columnwidth]{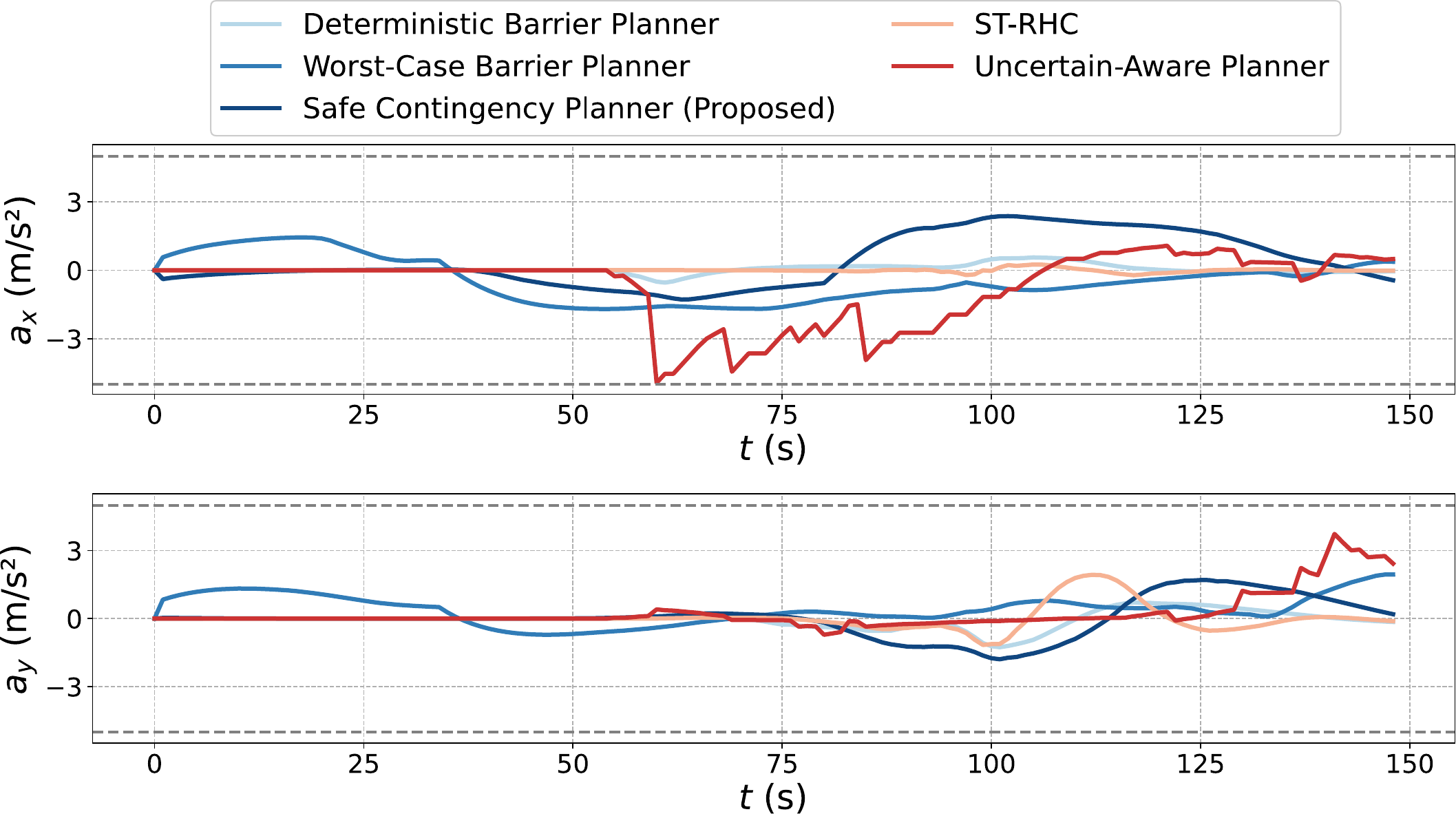}    \vspace{-5mm}
\caption{  Evolution of longitudinal and lateral acceleration profiles of EV with different planning methods (bounds shown as dashed lines). Our method shows proactive deceleration followed by smooth lateral avoidance during HV-A's abrupt cut-in maneuver.
} 
\label{fig:scene_1_acc}
\end{center}
\vspace{-4mm}
\end{figure} 

\begin{table}[t]
\caption{Performance Evaluation Across Key Metrics}
\label{tab:comparison}
\centering
\footnotesize
\setlength{\tabcolsep}{3pt}
\begin{tabular}{@{}lccccc@{\hspace{3pt}}cc@{}}
\toprule
\textbf{Method} & 
\multicolumn{2}{c}{\textbf{Safety}} & 
\multicolumn{2}{c}{\textbf{Comfort}} & 
\multicolumn{2}{c}{\textbf{Efficiency}} & 
\textbf{Time} \\
\cmidrule(lr){2-3}\cmidrule(lr){4-5}\cmidrule(lr){6-7}
& $P_\text{c}$ & $d_{\min}$ & $J_{x,\max}$ & $J_{y,\max}$ & $v_{\text{mean}}$ & $s_{\text{mean}}$ & $t_{\text{mean}}$ \\
& (\%) & (m) & (m/s\textsuperscript{3}) & (m/s\textsuperscript{3}) & (m/s) & (m) & (s) \\
\midrule
Deter. Barr. Pl. & 27.27 & 0.31 & 5.00 & 3.21 & 19.86 & 234.05 & 0.035 \\
Wors. Barr. Pl. & 0.00 & 6.53 & 6.17 & 7.08 & 15.75 & 186.07 & 0.033 \\
\textbf{Proposed} & 0.00 & 1.48 & 4.41 & 3.13 & 19.31 & 228.38 & 0.034 \\
ST-RHC \cite{zheng2024} & 18.18 & 0.35 & 6.39 & 3.83 & 20.21 & 239.49 & 0.029 \\
Uncert. Pl. \cite{zhou2025robustb} & 0.00 & 5.61 & 9.67 & 10.23 & 14.28 & 168.16 & 0.038 \\
\bottomrule
\end{tabular}
\vspace{-4mm}
\end{table}

% \begin{table}[t]
% \caption{\textcolor{blue}{Quantitative analysis of robustness to noise distribution.}}
% \label{tab:robust-ablations}
% \centering
% \footnotesize
% \setlength{\tabcolsep}{3pt}
% \begin{tabular}{@{}lccccc@{\hspace{3pt}}cc@{}}
% \toprule
% \textbf{Condition} & 
% \multicolumn{2}{c}{\textbf{Safety}} & 
% \multicolumn{2}{c}{\textbf{Comfort}} & 
% \multicolumn{2}{c}{\textbf{Efficiency}} & 
% \textbf{Time} \\
% \cmidrule(lr){2-3}\cmidrule(lr){4-5}\cmidrule(lr){6-7}
% & $P_\text{c}$ & $d_{\min}$ & $J_{x,\max}$ & $J_{y,\max}$ & $v_{\text{mean}}$ & $s_{\text{mean}}$ & $t_{\text{mean}}$ \\
% & (\%) & (m) & (m/s\textsuperscript{3}) & (m/s\textsuperscript{3}) & (m/s) & (m) & (s) \\
% \midrule
% $\sigma_1$ (cov.\,$\times$0.1) & 0.00 & 1.46 & 4.40 & 2.85 & 19.3 & 228.50 & 0.033 \\
% $\sigma_1$ (cov.\,$\times$0.5) & 0.00 & 1.47 & 4.37 & 4.88 & 19.3 & 228.66 & 0.033 \\
% Gaussian ($\sigma_0$) & 0.00 & 1.48 & 4.41 & 3.13 & 19.3 & 228.38 & 0.034 \\
% $\sigma_2$ (cov.\,$\times$5.0) & 0.00 & 2.00 & 6.14 & 4.45 & 16.4 & 194.11 & 0.034 \\
% $\sigma_2$ (cov.\,$\times$10.0) & 0.00 & 3.84 & 5.38 & 4.80 & 15.9 & 172.52 & 0.033 \\
% \midrule
% Laplace & 0.00 & 1.48 & 4.39 & 3.13 & 19.3 & 228.50 & 0.041 \\
% Uniform & 0.00 & 1.44 & 4.38 & 2.93 & 19.3 & 228.59 & 0.041 \\
% Cauchy & 0.00 & 2.38 & 13.28 & 4.79 & 14.6 & 230.60 & 0.061 \\
% \bottomrule
% \end{tabular}
% \end{table}

\begin{table}[t]
\caption{Quantitative analysis of robustness to noise distribution and covariance mismatch.}
\label{tab:robust-ablations}
\centering
\footnotesize
\setlength{\tabcolsep}{3pt}
\begin{tabular}{@{}lccccc@{\hspace{3pt}}cc@{}}
\toprule
\textbf{Noise Condition} & 
\multicolumn{2}{c}{\textbf{Safety}} & 
\multicolumn{2}{c}{\textbf{Comfort}} & 
\multicolumn{2}{c}{\textbf{Efficiency}} & 
\textbf{Time} \\
\cmidrule(lr){2-3}\cmidrule(lr){4-5}\cmidrule(lr){6-7}
& $P_\text{c}$ & $d_{\min}$ & $J_{x,\max}$ & $J_{y,\max}$ & $v_{\text{mean}}$ & $s_{\text{mean}}$ & $t_{\text{mean}}$ \\
& (\%) & (m) & (m/s\textsuperscript{3}) & (m/s\textsuperscript{3}) & (m/s) & (m) & (s) \\
\midrule
\textit{Covariance Mismatch} & & & & & & & \\ % 可选：加一个小标题或者直接列出
Gaussian ($0.1\times\Sigma_0$) & 0.00 & 1.46 & 4.40 & 2.85 & 19.3 & 228.50 & 0.033 \\
Gaussian ($0.5\times\Sigma_0$) & 0.00 & 1.47 & 4.37 & 2.98 & 19.3 & 228.66 & 0.033 \\
Gaussian ($\Sigma_0$) & 0.00 & 1.48 & 4.41 & 3.13 & 19.3 & 228.38 & 0.034 \\
Gaussian ($5.0\times\Sigma_0$) & 0.00 & 2.00 & 5.38 & 4.45 & 16.4 & 194.11 & 0.034 \\
Gaussian ($10.0\times\Sigma_0$) & 0.00 & 3.84 & 6.14 & 4.80 & 15.9 & 188.17 & 0.033 \\ %172.52
\midrule
\textit{Distribution Type} & & & & & & & \\
Laplace & 0.00 & 1.48 & 4.39 & 3.13 & 19.3 & 228.50 & 0.033 \\
Uniform & 0.00 & 1.44 & 4.38 & 2.93 & 19.3 & 228.59 & 0.033 \\
Cauchy & 0.00 & 2.38 & 7.28 & 4.79 & 14.6 & 172.52 & 0.034 \\ %230.60
\bottomrule
\end{tabular}
\vspace{8pt}
\parbox{\linewidth}{\footnotesize \textit{Note:} The parameters of non-Gaussian distributions are configured to align the noise level with the nominal Gaussian baseline $\Sigma_0$ for fair comparison.}
\vspace{-17pt}
\end{table}

\textbf{1) Highway Cut-in Scenario on NGSIM.}

\begin{figure*}[tbp]
    \centering \hspace{-4mm}
    % First row of subfigures
    \subfigure[]{
        \includegraphics[scale=0.2435]{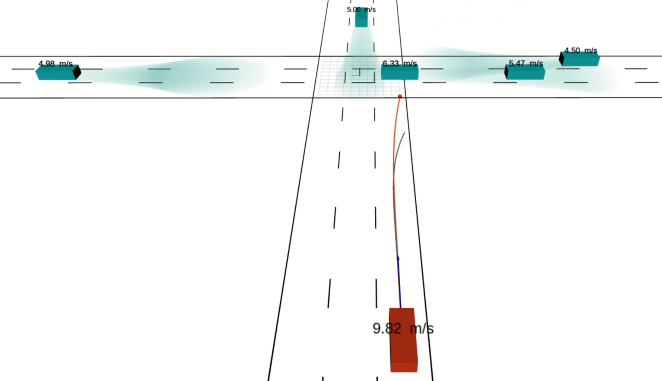}
    }
    \subfigure[]{
        \includegraphics[scale=0.2435]{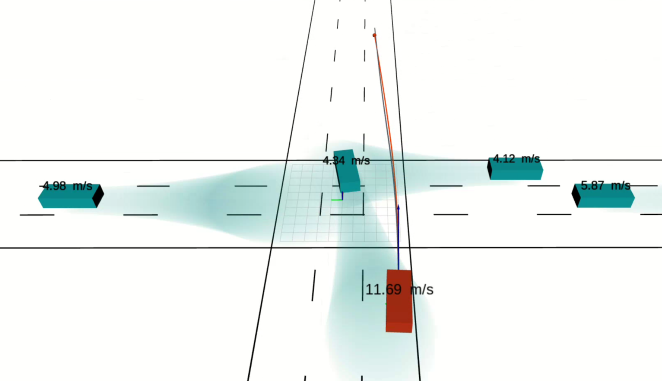}
    }
    \subfigure[]{
        \includegraphics[scale=0.2435]{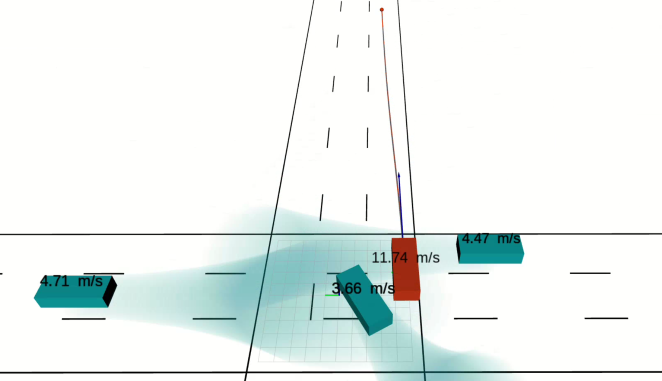}
    }
    \subfigure[]{
        \includegraphics[scale=0.2435]{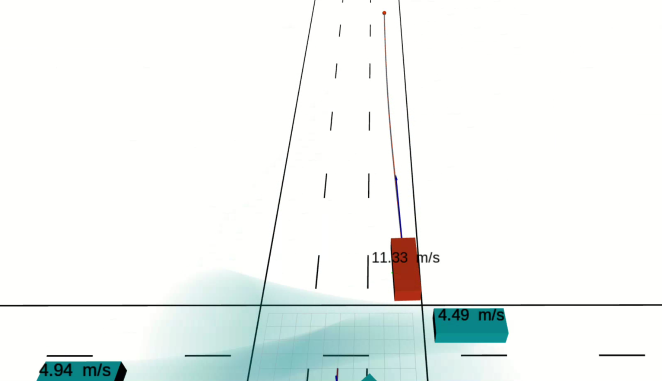}
    }
    \vspace{-4mm}
    \caption{Snapshots of the EV safely navigating through an unsignalized intersection. (a)-(d) demonstrate the EV's (red) response to a southbound HV's abrupt right turn while handling uncertainties from crossing traffic flows. Online updating HV FRSs are visualized as shaded regions. }
    \label{fig:intersection_snapshots}
    \vspace{-4mm}
\end{figure*}
% \textcolor{blue}{Note that the geometric overlap between the EV and FRSs arises from projecting time-varying spatiotemporal predictions onto a static 2D plane; the FRS-based barrier function $\mathscr{B}$ remains strictly positive throughout the simulation, indicating that safety constraints are consistently satisfied.}

\begin{figure}[tb]
\begin{center}
\includegraphics[width=1.0\columnwidth]{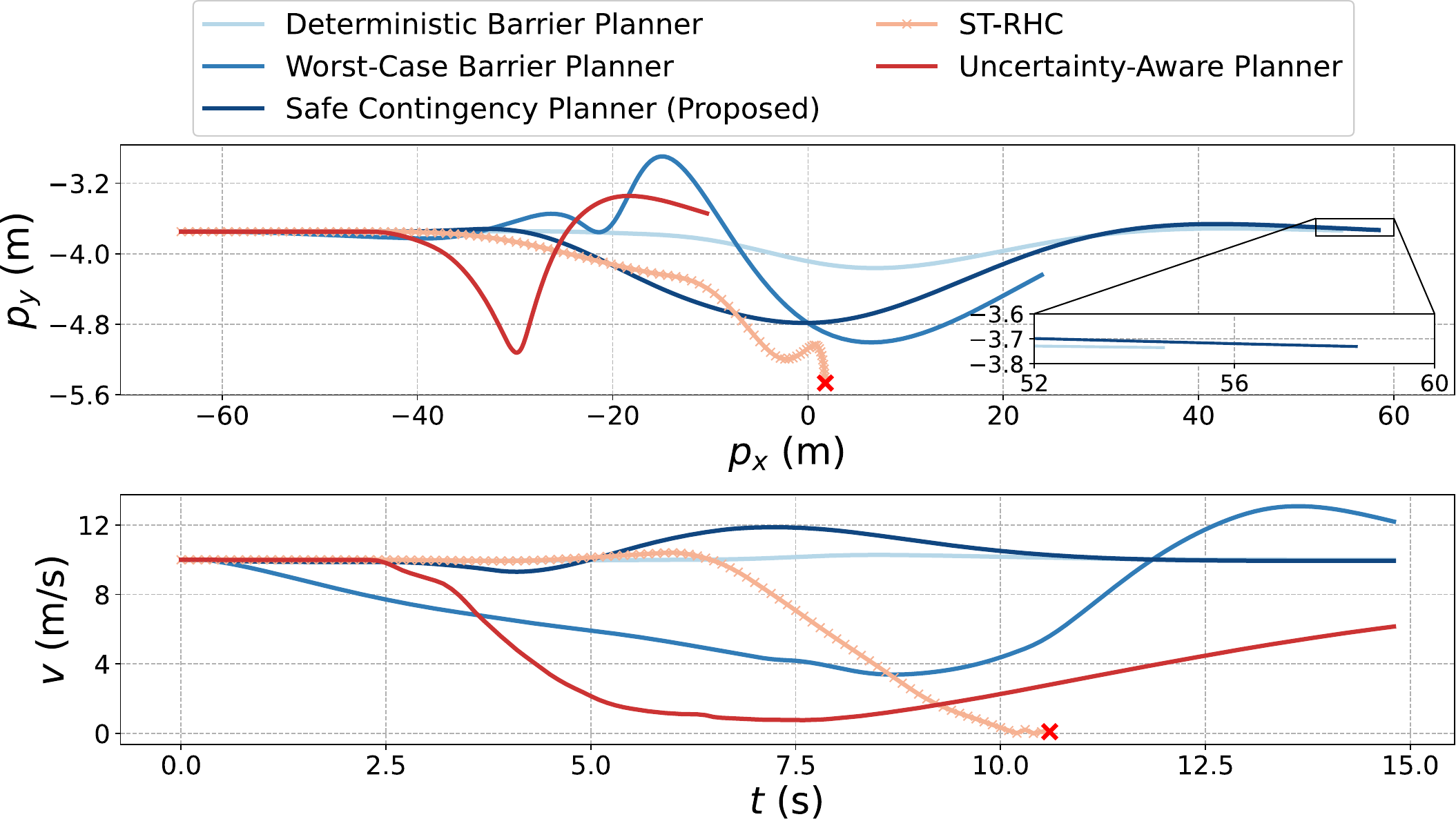}    \vspace{-4mm}
\caption{  
Trajectory comparison in the intersection scenario demonstrates coordinated defensive deceleration and proactive lateral adjustments using the proposed method, achieving the longest navigation distance. While both the Deterministic Barrier Planner and ST-RHC result in collision (with ST-RHC additionally suffering from planner infeasibility, indicated by red crosses), the Worst-Case Barrier Planner and Uncertainty-Aware Planner exhibit excessive caution with unnecessary deceleration.
} \vspace{-8mm}
\label{fig:intersection_p}
\end{center}
\end{figure} 
\begin{figure}[tb]
\begin{center}
\includegraphics[width=1.0\columnwidth]{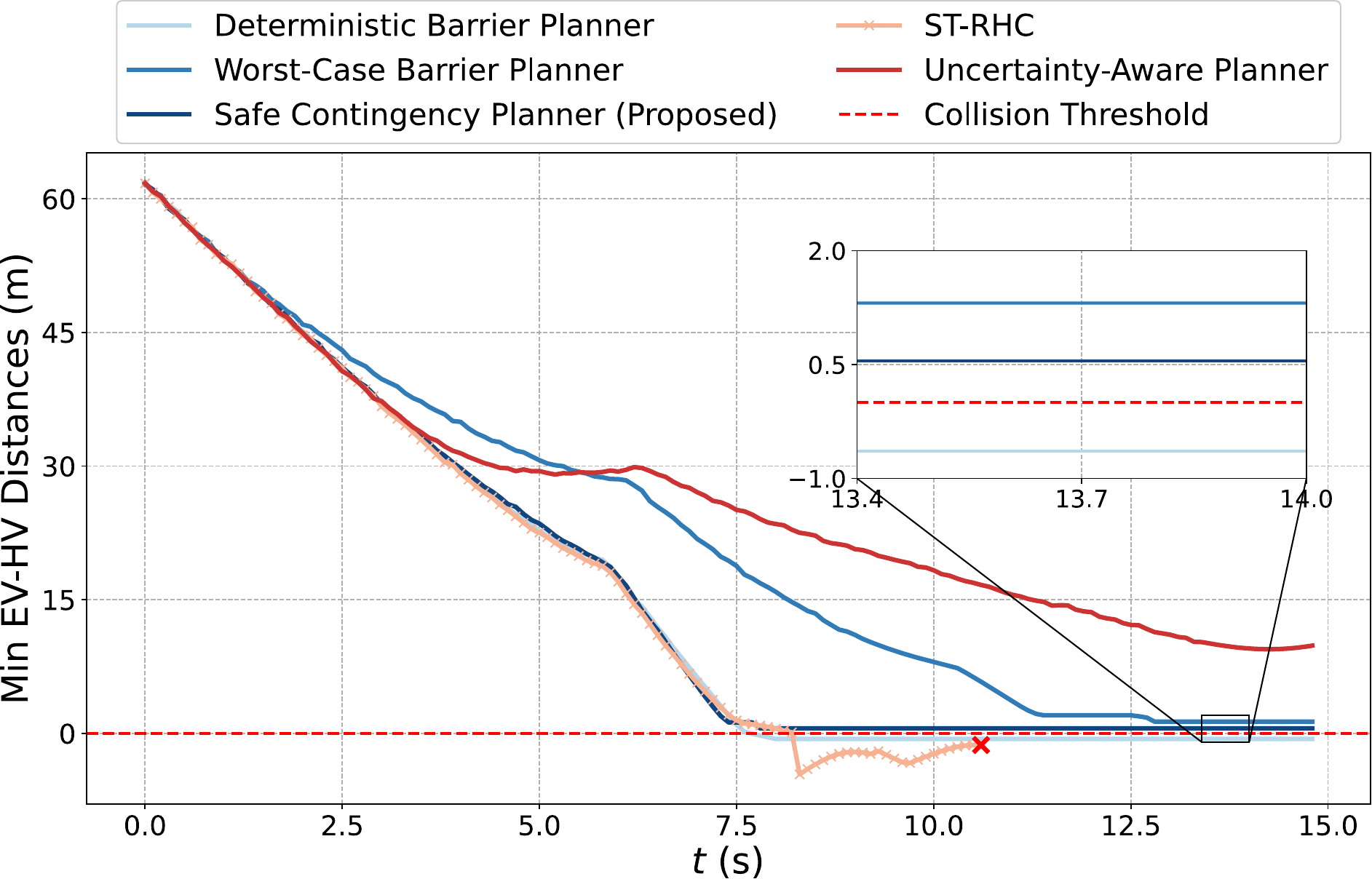}    \vspace{-4mm}
\caption{  
The minimum EV-HV distances ($d_{\text{min}}$) show that our method maintains appropriate safety margins. The Deterministic Barrier Planner and ST-RHC fail to avoid collision, while the Worst-case Barrier Planner and the Uncertainty-Aware planner exhibit unnecessarily conservative distances. 
} \vspace{-7mm}
\label{fig:intersection_min_dist}
\end{center}
\end{figure} 
The validation leverages challenging dense traffic driving scenarios from the NGSIM dataset to rigorously assess the performance. In the tested scenario, the EV navigates a multi-lane highway environment. As shown in Fig. \ref{fig:snapshots_cruise_static}, the HV-A in the leftmost lane initiates a rightward staged lane change, briefly pauses in the second leftmost lane, and then abruptly resumes its rightward maneuver just as the EV approaches. Meanwhile, the HV-B and HV-C maintain steady driving. This scenario challenges the EV's ability to dynamically identify and adapt to uncertain HV intentions while proactively responding to potential cut-ins.
% uncertain maneuver pattern of (characterized by an initial lateral movement, temporary pause, and sudden completion) 

The simulation configuration establishes an initial EV speed of $20\, \text{m}/\text{s}$, with time headways to HV-A systematically varied from $4.5\,\text{s}$ to $5.5\,\text{s}$ in $0.1\,\text{s}$ increments to ensure statistical reliability, with each parameter tested three times. Perception noise is modeled as a zero-mean Gaussian distribution \cite{zheng2024safe}. The standard deviations are modeled as decreasing functions of the EV-HV distance. For the noise in longitudinal position, the standard deviation is $\sigma_{p_x}^{h} = \overline{\sigma}_{p_x}/\max\left(\frac{10}{s^h + 0.1}, 1\right)$, where $s^h$ denotes the EV-HV distance and $\overline{\sigma}_{p_x}=0.2\, \text{m}$ is a noise parameter. The same principle applies to the lateral position, longitudinal and lateral velocity components, with the corresponding noise parameters set as \(\overline{\sigma}_{p_y}=0.2\, \text{m}\), \(\overline{\sigma}_{v_x}=0.1\, \text{m/s}\) and \(\overline{\sigma}_{v_y}=0.1\, \text{m/s}\), repsectively. These parameters define the baseline noise profile, denoted as the nominal environmental noise condition $\Sigma_0$. In the robustness analysis, we scale this baseline $\Sigma_0$ to simulate varying degrees of sensing quality, while the algorithm's internal parameters remain fixed.

The performance evaluation presented in Table \ref{tab:comparison} demonstrates that the proposed method achieves collision-free with a safe minimum EV-HV distance of $1.48 \,\text{m}$, while maintaining comfort requirements and competitive efficiency metrics. It represents significant improvements over both the Deterministic Barrier Planner with $27.27\%$ collision rate and ST-RHC with $18.18\%$ collision rate. This safety assurance stems from our rigorous FRS-based barrier constraints.

Trajectory analysis in Fig. \ref{fig:scene_1_trajectory} and  Fig. \ref{fig:scene_1_acc} reveals distinct behavioral patterns among the compared methods. The proposed approach exhibits defensive characteristics through gradual deceleration initiated at approximately $40\,\text{s}$ to anticipate potential cut-in maneuvers by HV-A. This proactive strategy enables smooth speed adjustment combined with moderate lateral avoidance when HV-A executes its sudden lane change, ensuring safe passage. In contrast, the Worst-Case Barrier method exhibits overly conservative behaviors, implementing premature, exaggerated deceleration and lateral displacement, substantially compromising driving efficiency. Meanwhile, the ST-RHC, relying exclusively on deterministic predictions, resorts to last-minute aggressive avoidance maneuvers when confronted with actual cut-in situations. Notably, our method demonstrates significant comfort improvements, reducing peak longitudinal and lateral jerk by 31.0\% and 18.3\% respectively compared to ST-RHC. 

Moreover, comparative analysis with the Uncertain-Aware Planner \cite{zhou2025robustb} shows superior performance of our method in both driving efficiency and comfort. While this baseline approach directly employs FRS for obstacle avoidance, its lack of contingency planning capability results in abrupt lateral maneuvers and degraded driving efficiency.  The discontinuous nature of FRS updates leads to sudden trajectory adjustments, manifesting as high lateral jerk. Our method addresses these limitations through integrating FRS-based barrier constraint in contingency planning, achieving a 64.4\% reduction in lateral jerk while maintaining smoother acceleration profiles within constraints. 

Critically, the comparison serves as a sensitivity study on the initialization of control-intent set. Even with underestimated initialization, our method maintains zero collisions by rapidly adapting to aggressive behaviors, whereas the Worst-Case Barrier method with an overestimated control-intent set achieves safety at the cost of high conservatism. This confirms that the framework remains valid under overestimation while avoiding conservatism through adaptation when underestimated.

% \textcolor{blue}{To further assess the planner’s adaptability under dense-traffic conditions, we examined a case where the EV was flanked by two adjacent HVs in parallel lanes. In denser traffic where adjacent HVs (e.g., HV-A and HV-C) flank the EV, the planner expands their FRSs when lateral uncertainty increases, causing the EV to decelerate defensively; conversely, under stable lane keeping, compact FRSs allow the EV to proceed smoothly through the gap.}

To assess the practical robustness against non-Gaussian or covariance-mismatched Gaussian noise, we conducted extensive sensitivity analyses. As summarized in Table II, the framework maintains a $0\%$ collision rate across all scenarios. Notably, the proposed method inherently accommodates perception uncertainty: under heavy-tailed noise (e.g., Cauchy), the planner adaptively expanded the safety margins (increasing $d_{\min}$ from $1.48$ m to $2.38$ m) and reduced speed, effectively trading off efficiency for safety without parameter retuning.

% From a computational perspective, the proposed approach achieves real-time performance with $0.034 \,\text{s}$ per iteration. Biconvex structure exploitation with ADMM enables reliable real-time optimization for autonomous driving.

\begin{table}[t]
\caption{Sensitivity Analysis of Barrier Parameter $\alpha$}
\label{tab:alpha_sensitivity}
\centering
\footnotesize
\setlength{\tabcolsep}{3.5pt}
\begin{tabular}{@{}lccccccc@{}}
\toprule
\textbf{Param.} & 
\multicolumn{2}{c}{\textbf{Safety}} & 
\multicolumn{2}{c}{\textbf{Comfort}} & 
\multicolumn{2}{c}{\textbf{Efficiency}} & 
\textbf{Time} \\
\cmidrule(lr){2-3}\cmidrule(lr){4-5}\cmidrule(lr){6-7}
$\alpha$ & $P_\text{c}$ & $d_{\min}$ & $J_{x,\max}$ & $J_{y,\max}$ & $v_{\text{mean}}$ & $s_{\text{mean}}$ & $t_{\text{mean}}$ \\
& (\%) & (m) & (m/s\textsuperscript{3}) & (m/s\textsuperscript{3}) & (m/s) & (m) & (s) \\
\midrule
0.2 & 0.00 & 1.511 & 4.48 & 3.40 & 19.2 & 227.79 & 0.034 \\
0.4 & 0.00 & 1.497 & 4.45 & 3.34 & 19.3 & 228.25 & 0.034 \\
0.6 & 0.00 & 1.489 & 4.42 & 3.35 & 19.3 & 228.24 & 0.034 \\
0.8 & 0.00 & 1.484 & 4.41 & 3.13 & 19.3 & 228.38 & 0.034 \\
1.0 & 0.00 & 1.474 & 4.39 & 3.13 & 19.3 & 228.48 & 0.035 \\
\bottomrule
\end{tabular}
\vspace{-6mm}
\end{table}

\textbf{2) Urban Intersection Scenario.}
To further rigorously evaluate the capability of handling emergent traffic conflicts under uncertainty, we conduct a high-risk unsignalized intersection scenario test. The EV navigates northbound at $10 \,\text{m/s}$ through dense east-west traffic ($6–10 \,\text{m/s}$) along $y$ axis. The critical challenge arises when an oncoming southbound HV along the $x$ axis initially exhibits normal driving behavior before abruptly executing a right turn across the intended path of the EV. This scenario specifically tests two key capabilities: (1) proactive defensive maneuvers to account for trajectory prediction inaccuracies, and (2)  simultaneous management of multi-vehicle motion uncertainties while ensuring safety.

The trajectory analysis in Figs.~\ref{fig:intersection_snapshots}--\ref{fig:intersection_p} shows safe navigation of our method through anticipatory deceleration and controlled rightward avoidance maneuvers, while maintaining minimum $0.55 \,\text{m}$ clearance from adjacent vehicles, as shown in Fig. \ref{fig:intersection_min_dist}. This defensive strategy achieves the longest travel distance ($122.68 \,\text{m}$) among all methods, optimally balancing safety and efficiency.

Comparative analysis reveals distinct limitations of baseline methods. The ST-RHC approach, relying on deterministic constant-velocity predictions, fails to anticipate the right-turning HV’s maneuver, resulting in a collision due to rapidly shrinking feasible space between adjacent vehicles, ultimately leading to planner infeasibility. Parallel deficiencies emerge in the Deterministic Barrier Planner, which fails to adequately handle uncertain HV behaviors without FRS-based barrier constraints. 

While the Worst-Case Barrier Planner ensures safety through excessive early deceleration and complete yielding, this approach severely compromises efficiency by requiring the EV to wait for complete intersection clearance. This result highlights the necessity of our online update mechanism for FRS predictions, adapting to observed behaviors rather than assuming worst-case behaviors. 
Similarly, the Uncertainty-Aware Planner with direct FRS constraint incorporation results in overly cautious behaviors with substantial speed reduction and the shortest travel distance. This limitation originates from directly embedding FRS constraints in the ST-RHC framework while lacking contingency optimization.

These findings collectively underscore the capability of our method to preserve safe contingency trajectories with feasibility in uncertain traffic conflicts.

% We evaluated the impact of initializing $C_0^h$ with underestimated (nominal) and overestimated (worst-case) sets. The results show that even with an underestimated initial set, the event-triggered online adaptation rapidly aligns the intent set with the HV's actual maneuvers, ensuring zero collisions. Conversely, an overestimated initialization yields safe but conservative behavior (higher $d_{\min}$, lower $v_{\text{mean}}$).

 % The comparative results collectively demonstrate that our approach achieves superior performance by balancing safety and efficiency through three key features proactive defensive maneuvers enabled by FRS-based occupancy prediction, and dynamic trajectory adaptation via event-triggered FRS updates.

\textbf{3) Discussion.}
The effectiveness of the proposed method is affected by two critical parameters: the consensus steps parameter $N_s$ and the weight $p_s$. These parameters collectively determine how the vehicle balances immediate responsiveness with long-term performance in uncertain traffic scenarios.

$N_s$ controls the consensus horizon shared by both trajectories. Fig.~\ref{fig:consensus_step} shows smaller values ($N_s \leq 10$) enable more aggressive maneuvers by permitting earlier commitment to specific trajectory, with fixed $p_s=0.5$. This can improve efficiency in most cases, but potentially reduce safety margins when HVs exhibit unexpected behaviors. 
Conversely, larger values ($N_s \geq 20$) enforce conservative behaviors by maintaining both deterministic and FRS-based constraints longer, often causing preemptive deceleration until intentions clarify. As demonstrated in Fig.~\ref{fig:consensus_step}, when $N_s$ exceeds a certain threshold, the planner conservatively decelerates until the HV's intention becomes clear.

$p_s$ regulates contingency branch weighting in optimization. Here we fix $N_s=5$. As shown in Fig.~\ref{fig:prob}, larger $p_s$ values prioritize FRS-based safety barrier constraints in consensus segments, producing earlier and more pronounced deceleration. In contrast, smaller $p_s$ values approximate deterministic planning that assumes perfect HV trajectory prediction, yielding reduced yet adequate safe margins.

For practical deployment, we recommend tailoring parameter configurations to specific operational requirements and desired conservatism levels. Although the method does not incorporate explicit risk quantification, it ensures safety when HV behaviors stay within the updated uncertainty bounds. Our analysis confirms that $N_s$ and $p_s$   effectively regulate the critical tradeoff between conservatism and flexibility in uncertain traffic scenarios, allowing methodical adjustment of planning behavior across the safety-efficiency spectrum.

Finally, a sensitivity analysis of the barrier coefficient $\alpha$ is summarized in Table \ref{tab:alpha_sensitivity}. It controls the adjustment rate of barrier function $\mathscr{B}$ during planning. While smaller $\alpha$ values impose stricter safety margins (resulting in a marginal increase in $d_{\min}$), and larger values relax the constraint stiffness (slightly reducing peak jerk), the overall efficiency and comfort metrics remain stable across the spectrum.

\begin{figure}[!t]
\centering
\includegraphics[width=0.97\linewidth]{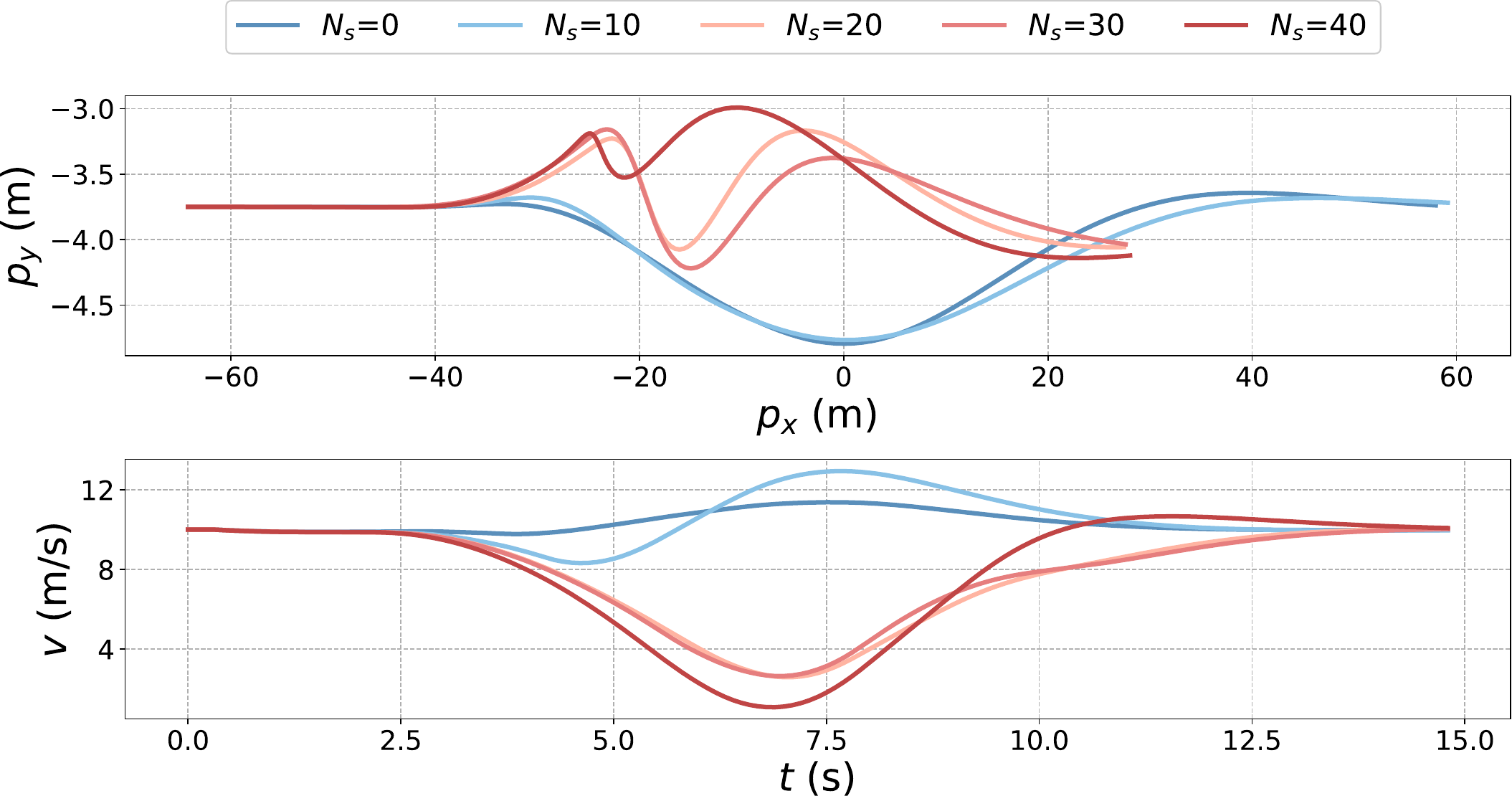} \vspace{-4mm}
\caption{Impact of the consensus step parameter $N_s$ on intersection navigation behaviors. Smaller values enable aggressive maneuvers while larger values enforce conservative strategies with preemptive deceleration.}
\label{fig:consensus_step}
\vspace{-3mm}
\end{figure}

\begin{figure}[!t]
\centering
\includegraphics[width=0.97\linewidth]{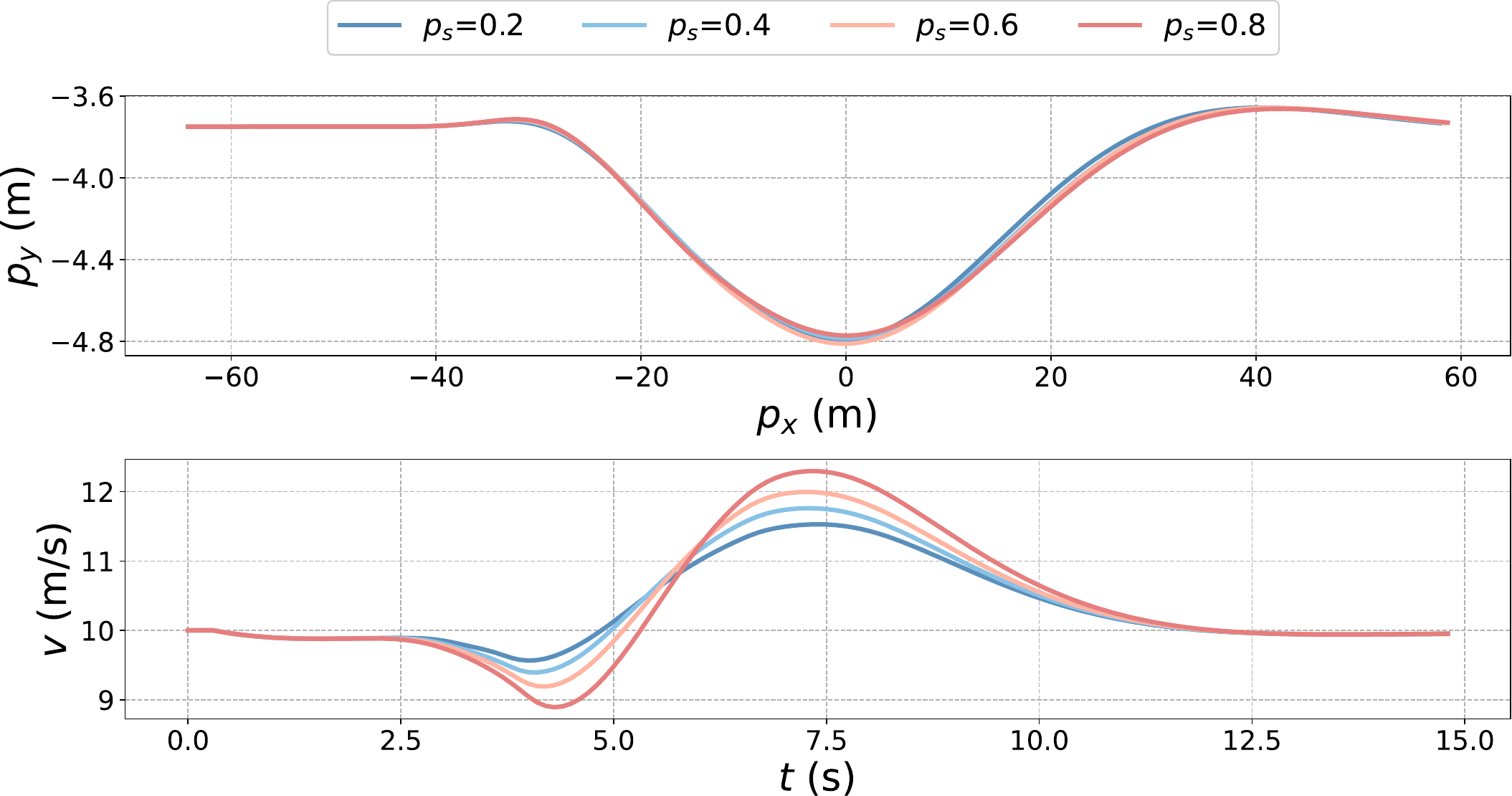} \vspace{-4mm}
\caption{Effect of the weighting parameter $p_s$ on intersection navigation behaviors. Higher values emphasize FRS-based safety constraints, resulting in more defensive behaviors.}
\label{fig:prob}
\vspace{-5.5mm}
\end{figure}

\vspace{-4mm}
\subsection{Experiment Results}\label{subsec:Exp}
To validate the practical efficacy of the proposed method, we conduct real-world experiments using 1:10-scale Ackermann-steering robotic platforms. The experimental setup consists of two robotic vehicles operating within an Optitrack motion capture system, with one serving as the EV and the other emulating the HV. The planning system, implemented on a laptop with computational specifications matching our simulation platform, receives state feedback through ROS Noetic via wireless communication and generates appropriate control commands. Trajectory tracking is achieved through a linear feedback controller that demonstrates sufficient tracking accuracy for our experimental purposes.

We design safety-critical overtaking scenarios to rigorously evaluate the adaptive responsiveness to evolving environmental uncertainties. These scenarios specifically assess its ability to generate safe yet nonconservative trajectories while maintaining robustness against unpredictable lane intrusions. In the experimental configuration, the EV navigates along the reference path at $p_y = 0\,\text{m}$ with a desired speed of $3\,\text{m}/\text{s}$ while attempting to overtake an HV moving from $p_y = -0.3$ m with varying degrees of aggressiveness. The experimental parameters are aligned with those used in the simulation.

\begin{figure}[tbp]
    \centering
    \subfigure[]{
        \includegraphics[width=0.97\columnwidth]{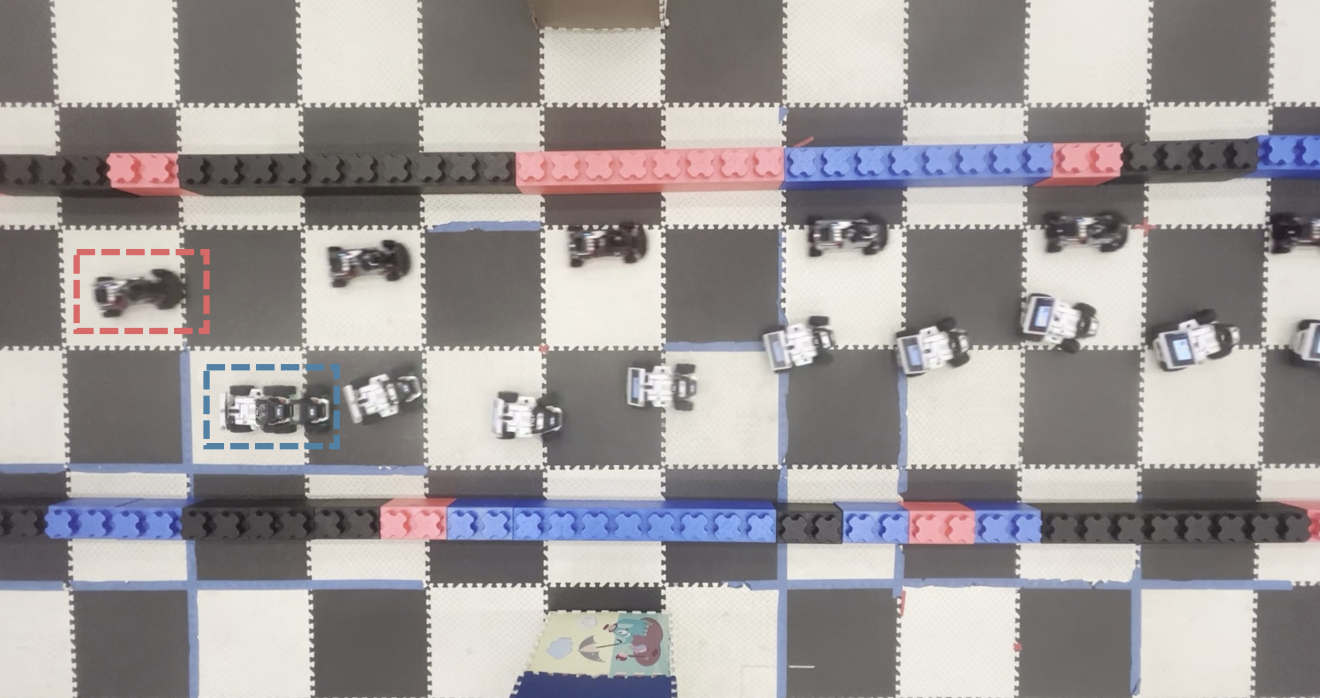} 
    }
    \subfigure[]{
        \includegraphics[width=0.97\columnwidth]{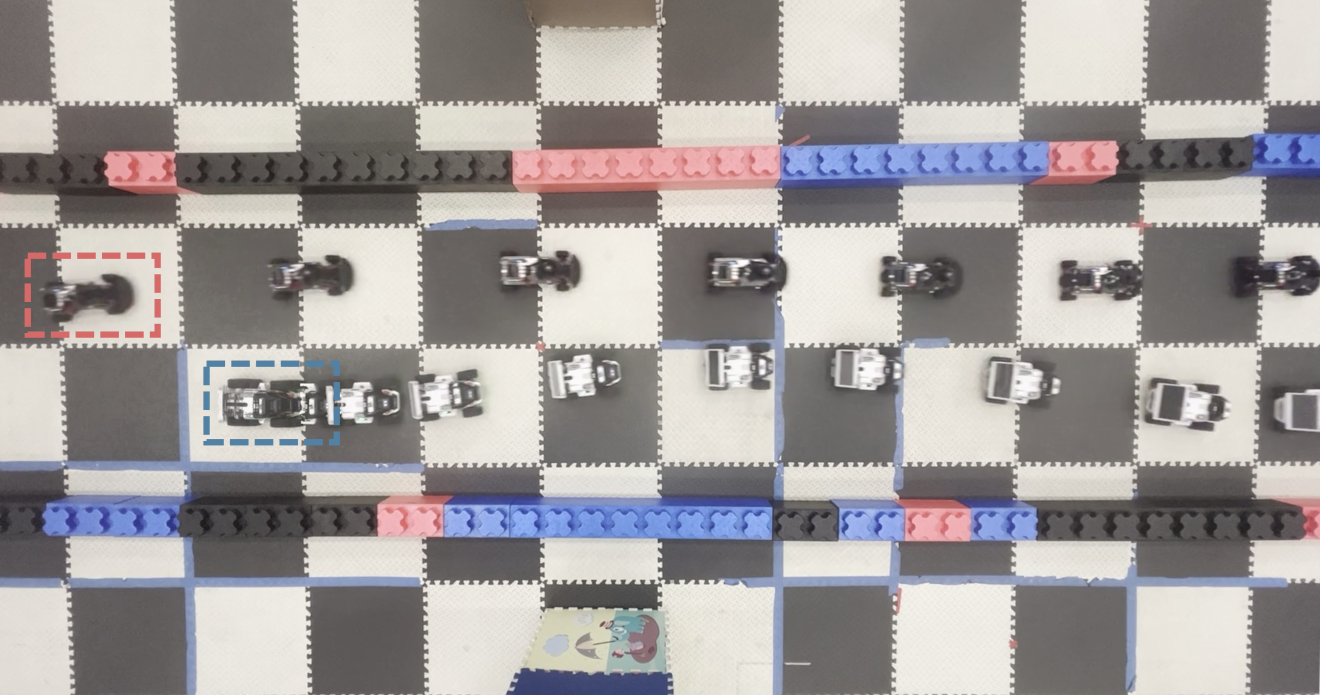}
    }
    \vspace{-3mm}
    \caption{Overlaying multiple frames of experiments showing (a) Scenario I: aggressive lane intrusion of HV, and (b) Scenario II: steady lane intrusion scenario of HV. The experiments highlight the adaptive safety assurance capabilities of the proposed method under varying HV uncertainties. The EV and HV are outlined with red and blue dashed boxes, respectively.}
    \label{fig:exp_snapshots}
    \vspace{-5mm}
\end{figure}

Fig. \ref{fig:exp_snapshots} shows the snapshots in two experimental scenarios that highlight the algorithm's performance under different HV uncertainties. In Scenario I, the HV exhibits aggressive, erratic behavior analogous to impaired driving, abruptly cutting into the EV's path at approximately $p_x = -2.8\, \text{m}$. Scenario II presents the contrasting case where the HV maintains relatively steady lane intrusion behavior. These scenarios encapsulate the fundamental safety challenges addressed by our method, particularly the EV's need to continuously estimate potential hazardous maneuvers without prior knowledge, dynamically update the HV's reachable sets for real-time safety assurance, and maintain viable contingency plans against possible lane intrusions.

The planned and executed trajectories are depicted in Fig. \ref{fig:exp_p}. In Scenario I, the HV's erratic motions, caused by intentional hardware modifications to induce unpredictable behavior, create significant uncertainty that prevents accurate trajectory prediction. Through online learning of the HV's control-intent set and subsequent FRS occupancy prediction, the EV successfully maintains both progress and safety. The resulting overtaking trajectory shows properly adaptive safety margin while effectively avoiding the HV's sudden lane intrusion towards $p_y = 0\, \text{m}$. Scenario II confirms the inherent adaptability of the proposed approach, where the relatively steady HV behavior results in smaller predicted reachable sets, enabling the planner to naturally generate more efficient overtaking trajectories with appropriately reduced conservatism.

Fig. \ref{fig:triggering} provides insight into the online adaptation mechanism governing the response to varying HV behaviors. The event $\|P_{u,m}^h u + q_{u,m}^h\|_2 \geq 1.0$ triggers the online learning of the control-intent sets given a newly collected control input vector $u$, leading to progressive expansion of set volume, computed by $\text{Vol}(\hat{\mathcal{U}}_m^h)=\pi\sqrt{\det(\Sigma_{u,m}^h)}$ \cite{boyd2004convex}. Scenario I demonstrates rapid growth of set volume corresponding to the HV's aggressive maneuvers, while Scenario II maintains smaller sets with relatively steady behaviors. This adaptive learning mechanism inherent in the FRS-based safety barrier enables the planner to automatically adjust safety margins according to the perceived driving uncertainty, providing a principled approach to balancing safety and efficiency.

\begin{figure}[tb]
\begin{center}
\includegraphics[width=1.\columnwidth]{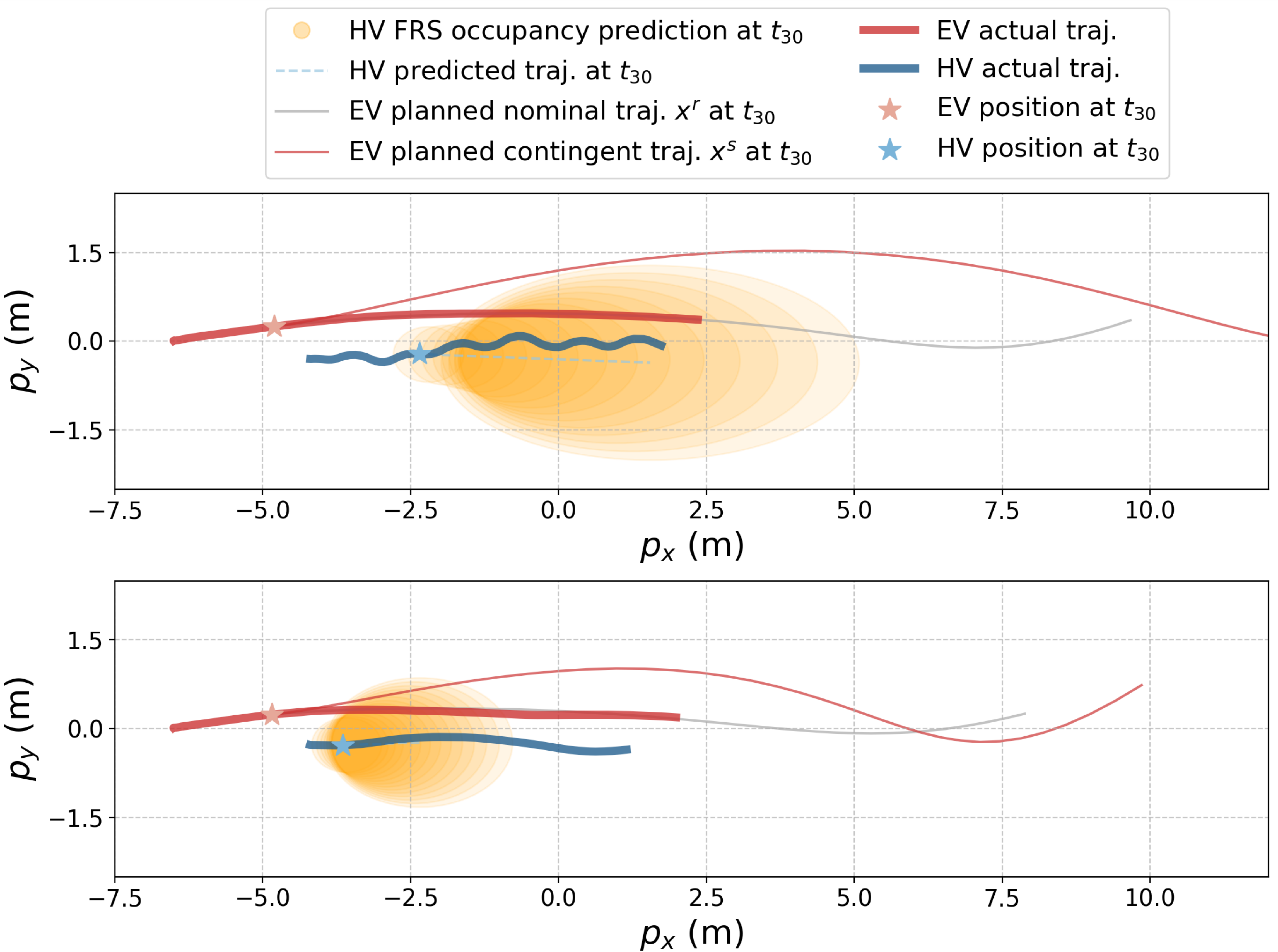}    \vspace{-5mm}
\caption{  
Overtaking trajectory comparison under different uncertainty levels: (Top) Scenario I with erratic HV shows larger safety margins against intrusion risks; (Bottom) Scenario II with steady HV demonstrates relieved overtaking maneuvers with reduced conservatism. Shaded regions indicate FRS predictions, showing adaptability to varying uncertainty conditions.
} 
\label{fig:exp_p}
\end{center}
\vspace{-8mm}
\end{figure} 

The physical experimental results demonstrate reliable and timely online update of varying FRS prediction through the event-triggered mechanism, which facilitates flexible adjustment of safety margins to match the perceived uncertain environment. Most importantly, the experimental outcomes verify that the planner generates nonconservative yet safe trajectories while maintaining defensive behaviors against sudden dangerous maneuvers, validating the practical applicability of our approach in safety-critical scenarios.

\vspace{-3mm}
%==================
\section{Conclusions}\label{sec:Conclusions}
\vspace{-1mm}
%==================
In this study, we present a real-time trajectory optimization framework that generates safe and nonconservative trajectories through FRS-based barrier constraints with event-triggered online adaptive refinement.
The method maintains feasible, safe trajectories under evolving uncertainties and avoids excessive conservatism or dangerous underestimation through incremental FRS updates. By decomposing the contingency optimization via consensus ADMM, the framework achieves computational efficiency for real-time implementation. Comprehensive validation demonstrates that the approach successfully achieves collision-free navigation with reduced jerks and maintains an average speed 14-20\% higher than other baselines, preserving safety with balanced driving efficiency in dense traffic scenarios. Future research will extend this framework by incorporating game-theoretic interaction models, which would lead to more natural cooperative behaviors in dense traffic.
\vspace{-5mm}

\begin{figure}[tb]
\begin{center}
\includegraphics[width=1.0\columnwidth]{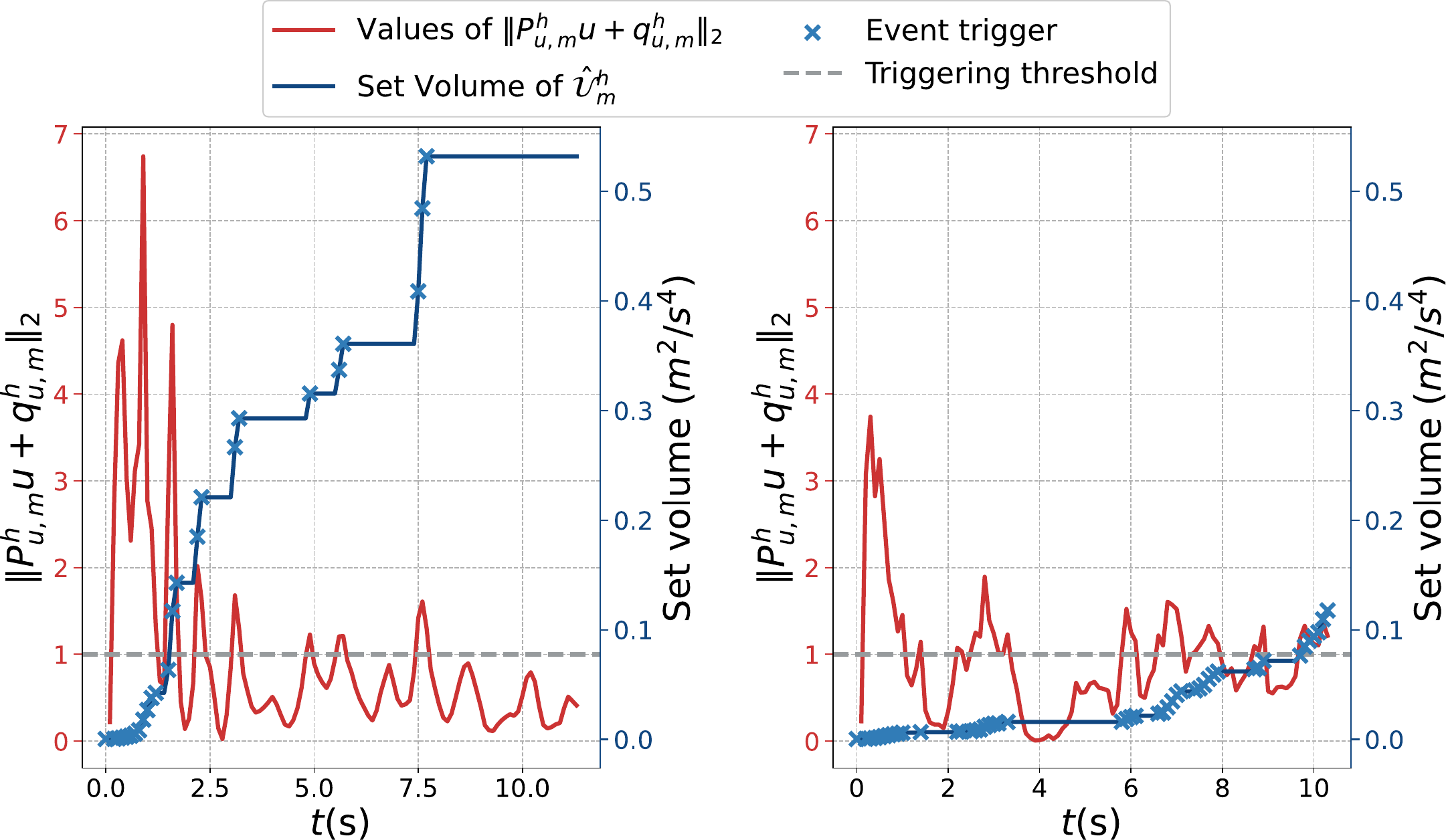}    \vspace{-5mm}
\caption{  
Evolution of control-intent set volume during online learning. The triggering condition prompts updates of control-intent set $\hat{\mathcal{U}}_m^h$, with Scenario I (left) showing a more rapid increase of set volume corresponding to aggressive behaviors than that in Scenario II (right).
} \vspace{-8mm}
\label{fig:triggering}
\end{center}
\end{figure} 
 
 % providing robust collision avoidance and improved passenger comfort against abrupt maneuvers.

\bibliographystyle{IEEEtran}
\bibliography{MyBibliography}

\vfill

\end{document}